\documentclass{article}
\pdfoutput=1
\usepackage[preprint]{neurips_2022}

\newcommand{\gdhBlue}[1]{{\textcolor{blue}{#1}}}

\usepackage[utf8]{inputenc} % allow utf-8 input
\usepackage[T1]{fontenc}    % use 8-bit T1 fonts
\usepackage{url}            % simple URL typesetting
\usepackage{booktabs}       % professional-quality tables
\usepackage{amsfonts}       % blackboard math symbols
\usepackage{nicefrac}       % compact symbols for 1/2, etc.
\usepackage{microtype}      % microtypography
\usepackage{xcolor}         % colors
\usepackage{amssymb}
\usepackage{float}
\usepackage{}
\newcommand{\comment}[1]{}
\usepackage{times}
\usepackage{soul}
\usepackage{url}
\usepackage[utf8]{inputenc}
\usepackage[small]{caption}
\usepackage{graphicx}
\usepackage{float}
\usepackage{amsmath}
\usepackage{amsthm,bm}
\usepackage{booktabs}
\usepackage[ruled,linesnumbered]{algorithm2e}
\usepackage{subfigure}
\usepackage{caption}
\usepackage{threeparttable} %为表格添加注释
\usepackage{multirow}
\usepackage[colorlinks,
linkcolor=blue,
anchorcolor=blue,
citecolor=blue]{hyperref}
\usepackage{wrapfig}
\usepackage{autobreak}

\usepackage{color}
\usepackage{colortbl}
\usepackage{enumerate}
\usepackage{authblk}
\usepackage{makecell}
\newcounter{counter}
\newtheorem{theorem}{Theorem}
\newtheorem{assumption}{Assumption}

\newtheorem{Lemma}[counter]{Lemma}

\newtheorem*{remark}{Remark}

\title{Fast Heterogeneous Federated Learning with \\Hybrid Client Selection}

\author[1,2]{Guangyuan Shen}
\author[\space\space2]{Dehong Gao\thanks{The first and second author contribute equally to this work. }}
\author[1]{Duanxiao Song}
\author[\space\space1]{Libin Yang\thanks{Contact Author}}
\author[1]{Xukai Zhou}
\author[3]{Shirui Pan}
\author[4]{Wei Lou}
\author[1]{Fang Zhou}
\affil[1]{Department of Cyber Science and Technology, Northwestern Polytechnical University, China}
\affil[2]{Alibaba Group, China}
\affil[3]{Department of Data Science and AI, Faculty of IT, Monash University, Australia}
\affil[4]{Department of Computing, The Hong Kong Polytechnic University, Hong Kong, China}
\affil[ ]{\texttt{\{gyshen, songduanxiao, libiny, zhouxukai, zhoufang\}@mail.nwpu.edu.cn,
dehong.gdh@alibaba-inc.com,
shirui.pan@monash.edu,
csweilou@comp.polyu.edu.hk}}
\begin{document}
\maketitle
\begin{abstract}
Client selection schemes are widely adopted to handle the communication-efficient problems in recent studies of Federated Learning (FL).
However, the large variance of the model updates aggregated from the randomly-selected unrepresentative subsets directly slows the FL convergence.
We present a novel clustering-based client selection scheme to accelerate the FL convergence by variance reduction.
Simple yet effective schemes are designed to improve the clustering effect and control the effect fluctuation, therefore, generating the client subset with certain representativeness of sampling.
Theoretically, we demonstrate the improvement of the proposed scheme in variance reduction.
We also present the tighter convergence guarantee of the proposed method thanks to the variance reduction.
Experimental results confirm the exceed efficiency of our scheme compared to alternatives.
\end{abstract}
\section{Introduction}
Federated Learning (FL) is a distributed learning paradigm for training a global model from data scattered across different clients~\cite{konevcny2015federated}.
During the training process, all the clients need to operate data locally and transfer the model updates between servers and themselves back and forth. 
Such a training process may  raise many challenges, with communication cost often being the critical bottleneck~\cite{kairouz2021advances}.

Many studies have found that different clients might transfer similar (or redundant) model updates to the server, which is a waste of communication costs~\cite{balakrishnan2022diverse,pmlr-v139-fraboni21a,karimireddy2020scaffold}.
Client selection schemes are widely adopted to reduce the waste of communication costs, e.g., Federated Averaging (\texttt{FedAvg}) where only the randomly-selected subset of clients transfer their model update to the server instead of yielding all clients involved~\cite{mcmahan2017communication:Lixto}. 
\texttt{FedAvg} can maintain the learning efficiency with reduced communication costs, since the random selection schemes can reduce the redundant update transmission.
When \texttt{FedAvg} meets client sets with low similarity, the server can not accurately pick out the representative clients.
The large variance of the update gradient aggregated from the unrepresentative subset directly slows the training convergence~\cite{pmlr-v139-fraboni21a,katharopoulos2018not}.

To accelerate the training convergence by variance reduction, many client selection criteria have been proposed in recent literature, e.g., importance sampling, where the probabilities for clients to be selected is proportional to their importance measured by the norm of update~\cite{chen2020optimal}, data variability~\cite{rizk2021optimal}, and test accuracy~\cite{mohammed2020budgeted}. 
However, importance sampling could not effectively capture the similarities among the clients.
As shown in Figure~\ref{Difference}\gdhBlue{(a)}, applying importance sampling could cause learning inefficiency as the clients transfer excessive important yet similar updates to the server.
To make better use of the similarities among clients, some researchers propose raw gradient-based cluster sampling schemes that group the clients with similar gradient together~\cite{pmlr-v139-fraboni21a,muhammad2020fedfast}, shown in Figure~\ref{Difference}\gdhBlue{(b)}.
If the clustering effect is good enough, cluster sampling can easily pick out the representative clients.
Unfortunately, the effect of raw gradient-based clustering methods faces the following problems:
\textbf{(i)~Poor Effect}, the high dimension gradient of a client is too complicated to be an appropriate cluster feature, which can not bring a good clustering effect.
\textbf{(ii) Effect Fluctuation}, due to the limitations of the clustering algorithm, when operating the client set with different parameters settings, e.g., different numbers of clusters, the clustering effect tends to fluctuate greatly.
After clustering with poor and fluctuant effect, there always exists some clusters with low client similarity.
Applying client selection in the clusters with low client similarity may slow the overall training convergence.
Besides, both the importance and cluster sampling methods require all the clients to return the raw gradient information for better selection, which runs against the communication reduction objective.
\begin{figure}[H]
    \centering
    \includegraphics[width=12.8cm]{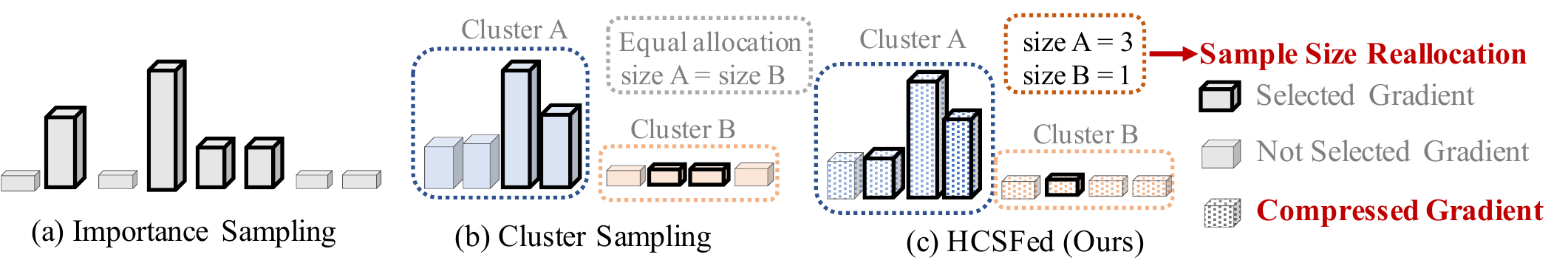}
    % \vspace{-1em}
    \caption{The differences among client selection schemes. The heights represent the norm of the gradients. The different colors denote different clusters. We highlight the main difference by red bold font.}
    \label{Difference}
\end{figure}
\vspace{-1em}
We propose \textbf{\texttt{HCSFed}}, a novel \underline{\textbf{H}}ybrid  \underline{\textbf{C}}lustering-based client \underline{\textbf{S}}election scheme for heterogeneous \underline{\textbf{Fed}}erated learning that further accelerates the FL convergence.
To improve the \textbf{poor clustering effect}, different from using the high dimension gradient, we adopt the compressed gradient as the cluster feature.
The different components of the gradient are grouped based on their numerical value, only the center of the group are retained to form the compressed gradient.
Adopting the compressed gradient as cluster feature not only allows better clustering effect but also is communication-efficient, since the compression operation filters out the redundant information within each gradient.
Besides, to control the \textbf{effect fluctuation}, we reset the number of clients to be selected in each cluster based on the client similarity and the size of the cluster, which we term as sample size re-allocation.
As shown in Figure~\ref{Difference}\gdhBlue{(c)}, sample size re-allocation scheme makes the cluster with low similarity have more number of clients to be selected.
The sampling effect becomes stable because the poor effect of random sampling in the clusters with low similarity is greatly mitigated by sample size re-allocation scheme.
Theoretically, we demonstrate the improvement of the proposed scheme in variance reduction.
Extensive experiments show the exceed efficiency of the proposed scheme compared to the cutting-edge FL client selection criteria in both convex and non-convex settings.
\section{System Setup and Prior Work}
\paragraph{The Federated Learning Optimization Settings.}
In federated learning, a total number of $N$ clients aim to jointly solve the following distributed optimization model:
\begin{eqnarray} \label{eq1}
\min_{\mathbf{w}\in \mathbb{R}^{d}}\left\{F(\mathbf{w}):=\sum^{N}_{k=1}\omega_k \left(F_k(\mathbf{w}):=\frac{1}{n_k}\sum^{n_k}_{j=1}f(\mathbf{w};x_{k,j})\right)\right\},
\end{eqnarray}
where $\omega_k$ is the weight of the $k^{th}$ client, $s.t.~\omega_k \ge 0$ and $\sum_{k=1}^N\omega_k=1$. 
Suppose the $k^{th}$ client possesses $n_k$ training data: $\{x_{k,1},~x_{k,2},~...~,~x_{k,n_k}\}$.
The local objective function $F_k(\mathbf{w})$ is defined as the average of $f(\mathbf{w};x_{k,j})$,  which is a user-specified loss function (possibly non-convex) made with model parameters $\mathbf{w}\in \mathbb{R}^{d}$ and training data $x_{k,j}$.
\paragraph{\texttt{FedAvg} Description.}
In the $t^{th}$ train round of the \texttt{FedAvg}, the server broadcasts the latest global model $\mathbf{w}_{t}$ to all the clients. 
The $k^{th}$ client sets  $\mathbf{w}_{t}^{k} \leftarrow \mathbf{w}_{t}$ and performs local training for $E$ epochs:
\begin{equation} \label{eq2}
    \mathbf{w}_{t+1}^{k} \longleftarrow \mathbf{w}_{t}^{k}-\eta_{t} \nabla F_{k}\left(\mathbf{w}_{t}^{k}, \xi_{t}^{k}\right),
\end{equation}
where $\eta_{t}$ is the learning rate and $\xi_{t}^{k}$ is a data sample uniformly selected from the local data.
In the communication round, the server aggregates the local model update to produce the new global model $\mathbf{w}_{t+E}$. 
In fact, to reduce the communication cost, \texttt{FedAvg} randomly generates a subset $\mathcal{S}_t$ consisting of $m$ clients from the entire client set and aggregates the model update, $\{\mathbf{w}_{t+E}^{1}, \mathbf{w}_{t+E}^{2}, \cdots, \mathbf{w}_{t+E}^{m}\}$, 
\begin{equation}
    \mathbf{w}_{t+E} \longleftarrow \underbrace{\sum_{k=1}^{N} \omega_{k} \mathbf{w}_{t+E}^{k}}_{\mathrm{Full~Participate}}\longleftarrow \underbrace{\sum_{k \in \mathcal {S}_{t}} \frac{N}{m}\omega_{k} \mathbf{w}_{t+E}^{k}}_{\mathrm{Partial~Participate}}.
\end{equation}
\texttt{LAG}~\cite{chen2018lag} is a improved aggregation scheme where it triggers the reuse of outdated gradients for the non-selected clients, but the variance of the aggregated gradient is still large, which directly slows the FL convergence.
To accelerate the convergence by variance reduction, many client selection criteria have been proposed in recent literature, including importance-based and clustering-based methods.
\vspace{-1em}
\paragraph{Importance Sampling.}Importance sampling is a non-uniform client selection method, where the probabilities for clients to be chosen are proportional to their importance~\cite{katharopoulos2018not,zhao2015stochastic}.
Different methods to measure the importance are proposed, including data variability~\cite{rizk2021optimal}, norm of update~\cite{chen2020optimal}, test accuracy~\cite{mohammed2020budgeted}, local rounds~\cite{singh2019sparq}, and local loss~\cite{cho2020client}.
Such methods can yield better selection by activating the clients with more importance.
However, they could not capture the similarities among the updates of clients and, therefore, trigger the communication redundancy.
\vspace{-1em}
\paragraph{Cluster Sampling.}To make better use of similarities among clients, some researchers propose raw gradient-based cluster sampling schemes that group the clients with similar gradients together~\cite{pmlr-v139-fraboni21a,muhammad2020fedfast}.
However, the effect of cluster sampling strongly relies on the performance of the clustering process.
Unfortunately, the high dimension gradient of a client is too complicated to be an appropriate cluster feature, which can not bring a good clustering effect.
Besides, meticulous one-by-one client selection~\cite{balakrishnan2022diverse,dieuleveut2021federated} is introduced to FL. It could achieve better selection outcomes but with extra expensive computation and communication costs for solving each noncommittal submodular maximization.
\vspace{-1em}
\paragraph{Difference from \texttt{SCAFFOLD}.}
While \texttt{SCAFFOLD}~\cite{karimireddy2020scaffold} appears to be similar to our method, there are fundamental differences.
\texttt{SCAFFOLD} augments the updates with extra “control-variate” item that is also transmitted along with the updates to reduce the variance, while we reduce the variance by selecting a more representative subset. 
The former focuses on update control, while the latter focuses on better update selection, indicating that they are orthogonal and compatible with each other.

\section{\texttt{HCSFed}: Proposed Client Selection Scheme in Federated Learning} 
% %
% In Section 3.1, we first stress that variance is the dominating factor that impacts convergence speed and then present an optimization objective with several constraints. 
% %
% In Section 3.2, a novel \texttt{HCSFed} client selection scheme is proposed with the promise of lower variance, naturally faster convergence. 
\subsection{Analysis for the Optimization Objective of Client Selection Scheme}\label{sec31}
Inspired by the convergence analysis work~\cite{katharopoulos2018not}, we define the convergence speed as follows,
\begin{equation}\label{eq4}
    \begin{split}
       \mathrm{Speed}\triangleq&-\mathbb{E}\left[\left\|\mathbf{w}_{t+1}-\mathbf{w}^{*}\right\|_{2}^{2}-\left\|\mathbf{w}_{t}-\mathbf{w}^{*}\right\|_{2}^{2}\right]\\
        =&-\mathbb{E}\left[\left(\mathbf{w}_{t}-\eta_{t} \nabla F_{k}\right)^{T}\left(\mathbf{w}_{t}-\eta_{t} \nabla F_{k}\right)+2 \eta_{t} \nabla F_{k}^{T} \mathbf{w}^{*}-\mathbf{w}_{t}^{T} \mathbf{w}_{t}\right] \\
        =& \underbrace{\ 2\eta_{t}\left(\mathbf{w}_{t}-\mathbf{w}^{*}\right) \mathbb{E}\left[G_{t}\right]-\eta_{t}^{2} \mathbb{E}\left[G_{t}\right]^{T} \mathbb{E}\left[G_{t}\right]}_{Constant}-
        \underbrace{\eta_{t}^{2} \operatorname{Tr}\left(\mathbb{V}\left[G_{t}\right]\right)}_{Changeable},
    \end{split}
\end{equation}
where $G_{t}=\nabla F\left(\mathbf{w}_{t}, \mathcal{S}_{t}\right)$ and $\mathbb{E}\left[G_{t}\right]=\sum_{k=1}^{N} \nabla F_{k}\left(\mathbf{w}_{t}^{k}, x^{k}\right)$ is the unbiased estimate of the true update.
The first two terms are constant, since $\mathbb{E}\left[G_{t}\right]$ has a constant value. 
Therefore, gaining a convergence speedup is equivalent to reducing the variance of gradient, i.e., $\mathbb{V}\left[{G}_{t}\right]$.
Consider a partial participation FL~\cite{li2020federated}, selecting $m$ clients can be regarded as selecting one client $m$ times by $m$ different discrete probability distributions $\left\{P_i\right\}_{i=1}^{m}$. 
For each distribution, we have $P_{i} = \left\{ p_{k}^{i}\right\}_{k=1}^N$, where $p_{k}^{i}$ denotes the probability for the $k^{th}$ client to be selected in $i^{th}$ distribution. 
For faster convergence, we seek to minimize the variance of model update $\mathbf{w}(\mathcal{S}_{t})$ aggregated from the subset that is equivalent to minimizing the variance of gradient corresponding to the following equation:
\begin{equation}
    \min_{p_{k}^{i}}\left\{
    \mathbb{V}\left[{G}_{t}\right]=\mathbb{V}\left[ \mathbf{w}(\mathcal{S}_t) \right]  \right\}, s.t., \sum_{k=1}^{N} p_{k}^{i}=1 \text { with } p_{k}^{i} \geq 0~,~\left\Vert \mathcal{S}_t \right\Vert=q \cdot N=m,
\end{equation}
where we use $q$ and $N$ to denote the sampling ratio and the total number of all clients, respectively.

\subsection{Proposed Client Selection Scheme \texttt{HCSFed} for Variance Reduction}
How to explore and make use of the similarities among clients with acceptable resource costs is the key issue in constructing client selection criteria.
We propose \texttt{HCSFed}, a client selection scheme that exploits three components (cluster sampling, sample size re-allocation, importance sampling) guaranteeing convergence speedup by variance reduction.
% clustering algorithm
\begin{wrapfigure}[18]{R}{0.50\textwidth}
	\flushright  % 右对齐环境
	 \vspace{-2.2em}  % 增加空行
\begin{minipage}{1\linewidth}  % 控制算法框大小
\SetInd{0.3em}{0.4em}  % 第一位调整竖线启示位置，第二位控制两条线之间的宽度
\IncMargin{1em} % 使行号向内缩进
\begin{algorithm}[H]
\caption{ClientClustering}
\label{alg:algorithm1}
\KwIn{compressed gradient of all the clients in the $t^{th}$ round $\{X_{t}^{k}\}_{k=1}^{N}$}
\KwIn{The number of the clusters $\mathrm{H}$}
% \textbf{Initialize} \{$X_{t}^{k}\}_{k=1}^{N} \leftarrow \mathrm{PCA}(\mathcal{K})$\;
% \For(\tcp*[f]{$\mathrm{H}$ denotes the target number of clusters}){$\mathrm{H}=\{3,\ldots,\log_{2}N\}$}{
\textbf{Initialize} Randomly select H clients as cluster centers $\{\boldsymbol{\mu}_{1},\boldsymbol{\mu}_{2},\ldots, \boldsymbol{\mu}_{\mathrm{H}}\}$\;
\textbf{Initialize} $C_{i}=\varnothing\ (1 \leq i \leq \mathrm{H})$\;
\Repeat{$\forall\ i = \{1,2,\ldots, \mathrm{H}\}, \boldsymbol{\mu}_{i}^{\prime}=\boldsymbol{\mu}_{i}$}
{\For{each client k=1,2,\ldots,N}
{$\lambda_{k} = \arg \min _{i \in \{1,2,\ldots,\mathrm{H}\}} \left\|{X}_{t}^{k}-\boldsymbol{\mu}_{i}\right\|_{2}$\; $C_{\lambda_{k}}=C_{\lambda_{k}}\bigcup\left\{k^{th}~\mathrm{client~with}~X_{t}^{k}\right\}$\;
}
\For{each cluster i=1,2,\ldots,$\mathrm{H}$}
{New center $\boldsymbol{\mu}_{i}^{\prime}=\frac{1}{\left|C_{i}\right|} \sum_{{X_{t}^{k}} \in C_{i}}{X_{t}^{k}}$\;
% \eIf{$\boldsymbol{\mu}_{i}^{\prime} \neq \boldsymbol{\mu}_{i}$}
{$\boldsymbol{\mu}_{i} \leftarrow \boldsymbol{\mu}_{i}^{\prime}$\;}
% {$\boldsymbol{\mu}_{i}^{\prime} = \boldsymbol{\mu}_{i}$\;}
}
}
% $S_{\mathrm{H}} \leftarrow$ compute the within-cluster sum of squares $\sum_{i=1}^{\mathrm{H}} \sum_{X_{t}^{k} \in C_{i}}\left|X_{t}^{k}-C_{i}\right|^{2}$\;
% }
% $\mathrm{H} \leftarrow {\arg \min}_{\mathrm{H} \in \{3,\ldots,\log_{2}N\}}\{S_{\mathrm{H}}\}$\;
\KwOut{$ \mathcal{G}=\left\{C_{1}, C_{2}, \ldots, C_{\mathrm{H}}\right\}$\;}
\end{algorithm}
\end{minipage}
\end{wrapfigure}
\begin{figure}[H]
    \vspace{-1em}
    \centering
    \includegraphics[width=13cm]{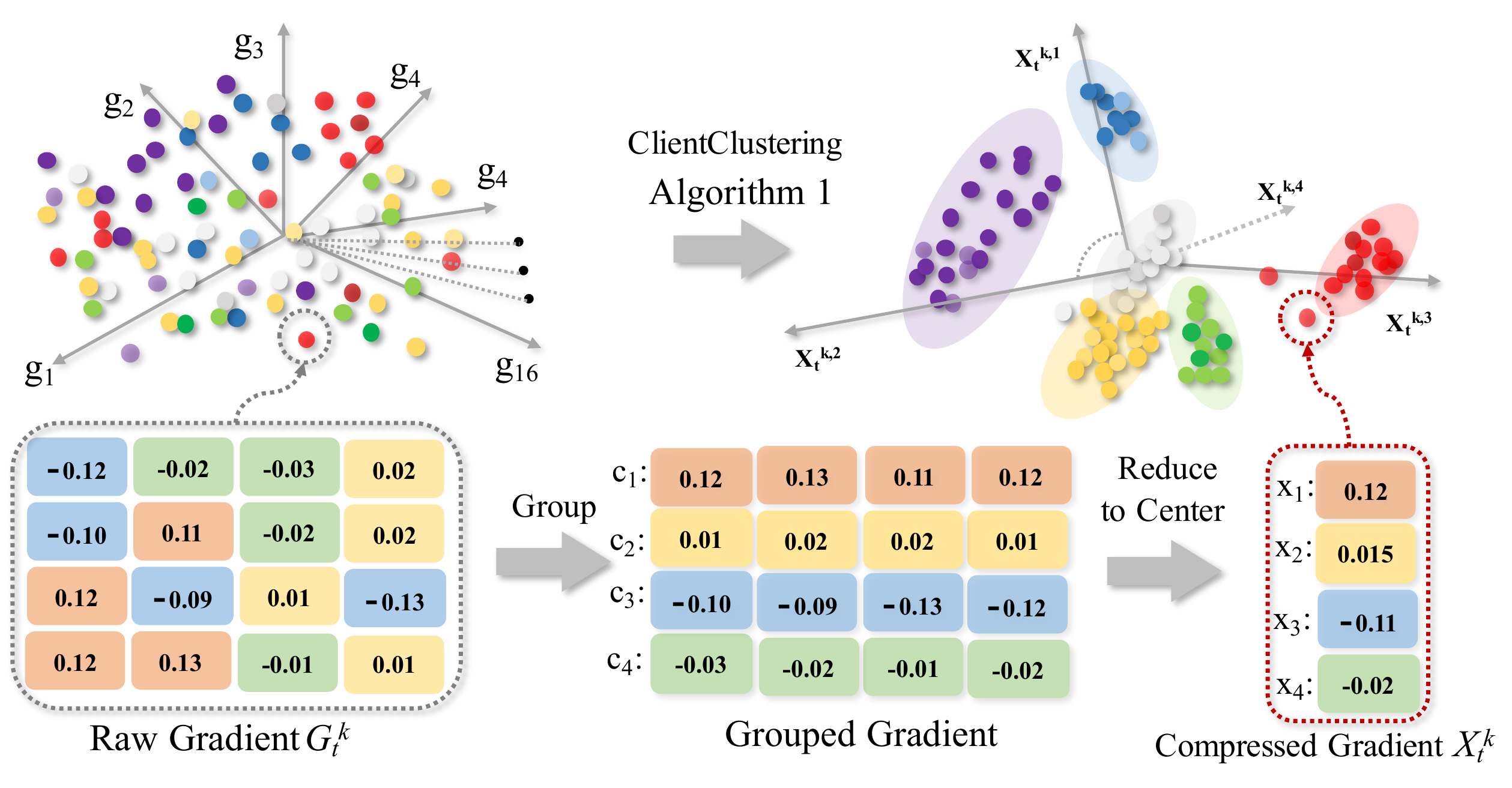}
    \caption{The visualizations of the raw gradient (left) and the clustered compressed gradient (right). Each point denotes the update gradient of one client. The $g_{i}$ is the $i^{th}$ component of the update gradient. The $c_{i}$ ci denotes the $i^{th}$ group of component. }
    \label{fig1}
\end{figure}
\vspace{-1em}
\paragraph{The Compressed Gradient-Based Cluster Selection.} 
To roughly explore and make use of the similarity among clients, we develop cluster selection that groups the similar clients together based on the client characteristic.
The following is a mathematical description of cluster selection. 
Generating a subset $\mathcal{S}_t$ consisting of $m$ clients by clustering selection can be regarded as selecting $m$ clients using $\mathrm{H}$ different probability distributions $\left\{P_h \right\}_{h=1}^\mathrm{H}$~where~$P_h = \left\{p_{k}^{h} \right\}_{k=1}^N$.
For example, if we select a client from the $h^{th}$ cluster, the probability distribution $P_h$ will be used, and the $p_{k}^{h}$ comes as follows:
\begin{equation} \label{eq5}
        p_{k}^{h} =
\begin{cases}
0, & \text{if }k^{th} \text{ client} \notin h^{th} \text{ cluster} \\
\frac{1}{N_{h}}, & \text{if }k^{th} \text{ client} \in h^{th} \text{ cluster}
\end{cases}
\end{equation}
where $N_h$ denotes the number of clients in the $h^{th}$ cluster $s.t. \sum_{h=1}^\mathrm{H} N_h = N$. 
Prior cluster selections~\cite{pmlr-v139-fraboni21a} prefer to develop clustering based on the raw high dimension gradient (shown in Figure~\ref{fig1}), which is too complicated to distinguish clients and brings unacceptable communication costs running against the original objective.
Different from the previous raw gradient-based clustering methods, we develop clustering based on compressed gradient \cite{han2015deep} which is the information-preserving representation of the model update. 
As shown in Figure~\ref{fig1}, different components of the gradient are grouped based
on their numerical value, only the center of the group are retained to form the compressed gradient.
The compression rate $R$ can be obtained by evaluating the quotient of $d$ and $d'$, i.e., $\frac{d'}{d}$, where $d'$ and $d$ are the total dimension of the compressed gradient and raw gradient, respectively.
Adopting the compressed gradient as cluster feature not only allows better clustering effect but also is
communication-efficient, since the compression operation filters out the redundant information within
each gradient.
%
% The pseudo-code of Gradient Compress (GC) can be found in Appendix A.2.
%
The pseudo-code of Gradient Compress (GC) can be found in Appendix \ref{app:A.2}.
Algorithm \ref{alg:algorithm1} summarizes the main steps of client clustering.
Ideally, after clustering, the clients in the same cluster are similar to each other.
However, as shown in Figure~\ref{fig1}, some clusters always exist with low similarity (purple cluster), since the effect of the clustering process fluctuates, which may also lead to performance degradation.

% 总算法
\begin{wrapfigure}[25]{R}{0.51\textwidth}
	\flushright  % 右对齐环境
	 \vspace{-2em}  % 负表示竖直方向向上移动
\begin{minipage}{1\linewidth}  % 控制算法框大小
\SetInd{0.3em}{0.4em}  % 第一位调整竖线起始位置，第二位控制两条线之间的宽度
\IncMargin{1em} % 使行号向内缩进
\begin{algorithm}[H]
\caption{\texttt{HCSFed}}
\label{alg:algorithm2}
\KwIn{updates in the $t^{th}$ round of all clients $\{G_{t}^{k}\}_{k=1}^{N}$}
\textbf{Initialize} $\mathbf{w}_{0}$\;
\textbf{Initialize} $\mathcal{S}_{t}=\varnothing$\;
\textbf{Server executes}:\\
%\definecolor{shadecolor}{rgb}{0.988,0.9215,0.843}
%\hspace*{-\fboxsep}\colorbox{shadecolor}
\For{each round $t = 1,2, \cdots$}
{$m \leftarrow \max(qN,1)$\;
$\mathcal{G} \leftarrow \text {ClientClustering}\left(\mathbf{X}_{t},\mathrm{H}\right)$\;
\For{each cluster $h=1,2,\cdots,\mathrm{H}$}
{$m_h \leftarrow \frac{N_{h} S_{h}}{\sum_{h=1}^\mathrm{H} N_{h} S_{h}} \cdot m$\;
% $\{p_{t}^{k}\}_{k \in N_{h}} \leftarrow\ \text {ImpSAMP}\left({N}_{h}, X_{t}\right)$\;
%
compute $\{p_{t}^{k}\}_{k \in N_{h}}$ using Equation \ref{eq7}\;
$\mathcal{S}_{t} = \mathcal{S}_{t} \bigcup m_{h}$ clients selected with $P_{h}$\;
% $\mathcal{S}_{t} = \mathcal{S}_{t} \bigcup m_{h}$ selected clients\;
}
\For{client $\in \mathcal{S}_{t}$ \textbf{in parallel}}
{$\mathbf{w}_{t+1}^k \leftarrow\;
\operatorname{ClientUpdate}(k,\mathbf{w}_{t}^{k})$\;
% $X^{k}_{t} \leftarrow \operatorname{ClientUpdate}(k,\mathbf{w}_{t}^{k}) $\;
}
$\mathbf{w}_{t+1} \leftarrow \sum_{{k} \in \mathcal{S}_{t}} \frac{N}{m} \omega_{k} \mathbf{w}_{t+1}^{k}$\;
}
\textbf{ClientUpdate($k, \mathbf{w}_{t}^{k}$)}: \\
%\definecolor{shadecolor}{rgb}{0.639,0.902,0.9255}
%\hspace*{-\fboxsep}\colorbox{shadecolor}
\For{each local epoch $i$ from $1\ \text{to}\ E$}
{$\mathbf{w}_{t+i+1}^{k}\longleftarrow\mathbf{w}_{t+i}^{k}-\eta \nabla F_{k}\left(\mathbf{w}_{t+i}^{k}, \xi_{t+i}^{k}\right)$\;}
% $\mathbf{w}^k_{t+1} \leftarrow \mathbf{w}^k_{t+E}$\;
% $X_{t}^{k} \leftarrow \mathrm{PCA}(G_{t}^{k})$\;
\eIf{$k \in \mathcal{S}_{t}$}{\textbf{return} $\mathbf{w}^k_{t+1} \leftarrow \mathbf{w}^k_{t+E}$,
$X_{t}^{k} \leftarrow \mathrm{GC}(G_{t}^{k})$\;}
{\textbf{return} $X_{t}^{k} \leftarrow \mathrm{GC}(G_{t}^{k})$\;}
% \textbf{return} $\mathbf{w^k_{t+1}}$\;
% \KwOut{The global model $\mathbf{w}_{t+1}^{k}$ broadcast to all clients}
\end{algorithm}
\end{minipage}
\end{wrapfigure}

\paragraph{Sample Size Re-allocation Scheme.}Concentrating on the diverse heterogeneity of clusters, we make the first attempt to develop a sample size re-allocation scheme that pays more attention to the cluster with great heterogeneity. 
\texttt{HCSFed} redetermines the sample size $m_{h}$, namely, the number of clients sampled from the $h^{th}$ cluster, by considering both the cluster size $N_h$ and variability. 
%
% \begin{equation}\label{eq6}
%     \begin{split}
%     m_h=\frac{N_{h} S_{h}}{\sum_{h=1}^\mathrm{H} N_{h} S_{h}} \cdot m,\\
%   \mathrm{where}~~S_h =\frac{1}{N_h}\sum_{j=1, i \neq j}^{N_h} \left\| X_i-X_j \right\|_{2},
%   \end{split}
% \end{equation}
\begin{equation}
    \begin{split}
        m_h=&\frac{N_{h} S_{h}}{\sum_{h=1}^\mathrm{H} N_{h} S_{h}} \cdot m, \\
        \mathrm{where}~~S_h =&\frac{1}{N_h-1}\sum_{j=1, i \neq j}^{N_h} \left\| X_t^{i}-X_{t}^{j} \right\|_{2}^{2},
    \end{split}
  \label{eq6}
\end{equation}
$m_h$ denotes the number of clients in the subset $\mathcal{S}_t$ from the $h^{th}$ cluster $s.t. \sum_{h=1}^\mathrm{H}m_h=m.$
$S_h$ denotes the variability of the $h^{th}$ cluster where we use Cluster Cohesion based on the compressed model update to approximate.
The introduction of sample size re-allocation can further reduce the variance by assigning more sample chance to those clusters with greater heterogeneity.
Due to the limited number of clients to be selected, sometimes some clusters with great heterogeneity still can not have an adequate sample chance, which indeed need extra care. 
\paragraph{Importance Selection.}
We introduce importance selection to optimize the probabilities for clients to be selected in the same cluster, which is based on the norm of the compressed gradient, i.e., $\mathrm{GC}(G_{t}^k)=X_{t}^{k}$.
We re-determine the probability for each client to be sampled in $t^{th}$ round as follows, 
% \begin{gather}
%     p_{t}^{k} = \frac{\left\|\mathrm{GC}(G_{t}^k)\right\|}{\sum_{k=1}^{N_{h}}\left\|\mathrm{GC}(G_{t}^k)\right\|}= \frac{\left\|X_{t}^k\right\|}{\sum_{k=1}^{N_{h}}\left\|X_{t}^k\right\|}, \nonumber\\ \label{eq7}
%     \text{ if } k^{th} \text{ client } \in h^{th} \text{ cluster in}~t^{th}~\text{round}.
% \end{gather}
\begin{equation} \label{eq7}
    p_{t}^{k} = \frac{\left\|\mathrm{GC}(G_{t}^k)\right\|}{\sum_{k=1}^{N_{h}}\left\|\mathrm{GC}(G_{t}^k)\right\|} = \frac{\left\|X_{t}^k\right\|}{\sum_{k=1}^{N_{h}}\left\|X_{t}^k\right\|},
    \text{if}\ k^{th}\ \text{client} \in h^{th}\ \text{cluster in}~t^{th}~\text{round}.
\end{equation}
After completing the importance selection, the client with a higher norm of gradient (more importance) will have a higher probability to be sampled, which brings representative outcomes under an inadequate sampling ratio.
Importance selection provides a fine-grained optimization over cluster selection with sample size re-allocation, i.e., assigning more attention to the more representative clients.

These three ideas of selection discussed in this subsection together constitute the \texttt{HCSFed} scheme to optimize the client selection in FL. 
\texttt{HCSFed} selects representative clients via allocating appropriate probability for each client to be selected, with the promise of reducing the variance between the model update aggregated from the sampled subset $\mathcal{S}_{t}$ and the entire set $\mathcal{K}$. 
As we mentioned in Section \ref{sec31}, the reduced variance leads to faster convergence.
Therefore, \texttt{HCSFed} can achieve communication cost reduction, as it requires fewer rounds to approach the target accuracy.
Algorithm~\ref{alg:algorithm2} summarizes the main steps of \texttt{HCSFed}.
\section{Theoretical Analysis}\label{sec4}
% We analyze the capability of our approach with regard to variance reduction.
% %
% Further, we show how \texttt{FedAvg} with stratified client selection can achieve faster convergence.
\subsection{Variance Relationship among Different Selection Schemes}
\begin{theorem}[Variance Reduction]\label{th1}
If the population is large compared to the subset, $\frac{m}{N}$, $\frac{m_h}{N_h}$, $\frac{1}{m_h}$ and $\frac{1}{N}$ are negligible, then the selection (or cross-client) variance of different selection schemes satisfies
$$
\mathbb{V}\left({\mathbf{w}}_{hybrid}\right) \leq \mathbb{V}\left({\mathbf{w}}_{cludiv}\right) \leq \mathbb{V}\left({\mathbf{w}}_{cluster}\right) \leq \mathbb{V}\left({\mathbf{w}}_{rand}\right),
$$
where $\mathbf{w}_{hybrid},\mathbf{w}_{cludiv},\mathbf{w}_{cluster},\mathbf{w}_{rand}$ denote the model update aggregated from the subset $\mathcal{S}_t$ that generated by \texttt{HCSFed}, clustering selection scheme under re-allocation, clustering selection scheme under plain allocation and simple random selection scheme, respectively.
\end{theorem}

\paragraph{Proof Sketch.} 
% Below we provide a proof sketch to reveal the relationship among the cross-client variance of different client selection schemes, and then we naturally state when our \texttt{HCSFed} scheme could achieve variance reduction. We defer the details of proof to Appendix B.
%
Below we provide a proof sketch to reveal the relationship among the cross-client variance of different client selection schemes, and then we naturally state when our \texttt{HCSFed} scheme could achieve variance reduction. We defer the details of proof to Appendix \ref{app:B}.

\textbf{Comparison of Random and Clustering Selection.}
We derive the Equation~\ref{eq9} to show the relationship between $\mathbb{V}\left({\mathbf{w}}_{cluster}\right)$ and $\mathbb{V}\left({\mathbf{w}}_{rand}\right)$:
\begin{equation}\label{eq9}
\mathbb{V}\left({\mathbf{w}}_{rand}\right) = \mathbb{V}\left({\mathbf{w}}_{cluster}\right) + \sum_{h=1}^{\mathrm{H}} \frac{N_h\left\|\mathbf{W}_{h} - \mathbf{W}(\mathcal{K})\right\|_{2}^2}{mN},
\end{equation}
where $\mathcal{K}$ denotes the set of all the clients. We have $\mathbb{V}\left(\mathbf{w}_{cluster}\right)<\mathbb{V}\left(\mathbf{w}_{rand}\right)$, unless $\forall h \in \{1, \cdots, \mathrm{H}\}, \mathbf{W}_{h} = \mathbf{W}(\mathcal{K}) $, i.e.,  each cluster has the same averaged model update with the entire population, which indicates that all the clusters are homogeneous in terms of the mean model update.

\textbf{Comparison of Plain Clustering and Clustering with Re-allocation.}
We derive the Equation~\ref{eq10} to show the relationship between $\mathbb{V}\left({\mathbf{w}}_{cludiv}\right)$ and $\mathbb{V}\left({\mathbf{w}}_{cluster}\right)$:
\begin{equation}\label{eq10}
\mathbb{V}\left({\mathbf{w}}_{cluster}\right) = \mathbb{V}\left({\mathbf{w}}_{cludiv}\right) + \sum_{h=1}^{\mathrm{H}} \frac{N_h({S_h} - {S})^2}{mN}.
\end{equation}
We have $\mathbb{V}\left(\mathbf{w}_{cludiv}\right) < \mathbb{V}\left(\mathbf{w}_{cluster}\right)$, unless $\forall h\in \{1, \cdots, \mathrm{H}\}, S_h={S}$, i.e., each cluster has equal variability, which indicates that all the clusters are homogeneous in terms of the variability.

\textbf{Comparison of Clustering with Re-allocation and \texttt{HCSFed}.}
We derive the Equation~\ref{eq11} to show the relationship between $\mathbb{V}\left({\mathbf{w}}_{cludiv}\right)$ and $\mathbb{V}\left({\mathbf{w}}_{hybrid}\right)$:
\begin{equation}\label{eq11}
\mathbb{V}\left({\mathbf{w}}_{cludiv}\right) = \mathbb{V}\left({\mathbf{w}}_{hybrid}\right) +\frac{\left(\sum_{i=1}^{N_{h}}\left\|G_{i}\right\|_{2}\right)^{2}}{N_{h}} \sum_{i=1}^{N_{h}}\left(I_{i}-\frac{1}{N_{h}}\right)^{2}.
\end{equation}
We have $\mathbb{V}\left(\mathbf{w}_{hybrid}\right) < \mathbb{V}\left(\mathbf{w}_{cludiv}\right)$, unless $\forall i\in \{1, \cdots, N_h\}, I_{i}=\frac{\left\|G_{i}\right\|_{2}}{\sum_{i=1}^{N_{h}} \left\|G_{i}\right\|_{2}}=\frac{1}{N_{h}}$, i.e., each client in $h^{th}$ cluster has equal norm of model update.
\begin{remark}
    Theorem \ref{th1} verifies the capability of \texttt{HCSFed} in variance reduction.
    The proof sketch indicates that the variance difference of different selections will vanish only if all the clusters and the norm of the gradients are identical.
    The proposed approach can achieve variance reduction, since each cluster never has identical update mean, variability, and norm in heterogeneous FL.
\end{remark}

\subsection{Convergence Behavior of \texttt{HCSFed}} 
% Before constructing the convergence analysis, we would like to ask a question: What exactly are we care about when we talk about the convergence of an algorithm?
% %
% The convergence analysis is mainly about giving an upper bound of the distance between the $F(\mathbf{w}_{t+1})$ and the global (or local) optimum $F(\mathbf{w}^{\star})$ , i.e., $F(\mathbf{w}_{T})-F(\mathbf{w}^{\star}) \leq E$. We stress that the $E$ mainly depends on the algorithm, but is sometimes influenced by the inequality that we use for scaling. 
% %
% However, comparing such convergence upper bound is unable to show how fast each algorithm converges, we are still willing to verify whether the FL loss function converges to a stationary point under our Hybrid selection. 
%
We emphasize that the convergence analysis of \texttt{FedAvg} with random selection is \textbf{NOT} our contribution, we just follow the previous analysis work~\cite{li2019convergence} to verify that \texttt{HCSFed} could maintain the same convergence guarantees as a random selection in the worst case and in most cases, \texttt{HCSFed} will have tighter convergence compared with random selection.

We next state the assumptions used in our theorem and proof, which are common in the federated optimization literature, e.g.,~\cite{zhou2017convergence,li2019convergence,wang2021cooperative,stich2018local,haddadpour2019convergence,haddadpour2019local}.
Assume each function $F_{k}(k \in[N])$ is $\mu$-strongly convex and $L$-smooth. 
Suppose that for all $k \in[N]$ and all $t$, the variance and expectation of stochastic gradients in each client on random samples $\xi$ are bounded by $\sigma_{k}^{2}$ and $G^2$, i.e., $\mathbb{E}\left[\left\|\nabla F_{k}\left(\mathbf{w}_{t}^{k}, \xi\right)-\nabla F_{k}\left(\mathbf{w}_{t}^{k}\right)\right\|_{2}^{2}\right] \leq \sigma_{k}^{2} $ and $\mathbb{E}\left[\left\|\nabla F_{k}^{2}\left(\mathbf{w}_{t}^{k}, \xi\right)\right\|_{2}^{2}\right] \leq G^{2}$, respectively. And $\Gamma$ is used to quantify the heterogeneity, where $\Gamma=F^{*}-\sum_{k=1}^{N} p_{k} F_{k}^{*}$.
\begin{theorem}[Convergence Bound]\label{th2}
Let Assumptions above hold and $L,\mu,\sigma_{k},G$ be defined therein. Consider \texttt{FedAvg} when sampling $m$ clients, then \texttt{HCSFed} satisfies:
\begin{align}
    \mathbb{E}\left[F\left(\mathbf{w}_{T}\right)\right]-F^{*} \leq \underbrace{\mathcal{O}\left(\frac{\sum_{k=1}^{N} p_{k}^{2} \sigma_{k}^{2}+ E^{2} G^{2}+\gamma G^{2}}{\mu T}\right)
    + \mathcal{O}\left(\frac{L \Gamma}{\mu T}\right)}_{Bounding~\texttt{FedAvg}~with~Full~Participation} +\underbrace{\mathcal{O}\left(\frac{ E^{2} G^{2}}{m\mu T}\right)}_{Bounding~Variance}.
\end{align}
\end{theorem}
\begin{remark}
    Theorem \ref{th2}~\cite{li2019convergence} gives the convergence upper bound of \texttt{FedAvg} with hybrid client selection.
    The first two terms are used to bound the \texttt{FedAvg} with full client participant, while the last term is used to bound the variance of the model update aggregated from the subset. In Theorem~\ref{th1}, we demonstrate the effectiveness of \texttt{HCSFed} in variance reduction. \texttt{HCSFed} achieves a lower variance, which indicates that \texttt{HCSFed} could enjoy a tighter convergence bound.
\end{remark}
\begin{remark}
    % Theorem \ref{th2}~\cite{li2019convergence} use Lemma 1-3 to bound the error of \texttt{FedAvg} with full participation corresponding to the first two terms, while using Lemma 4 to verify the unbiased property and Lemma 5 to bound the variance resulting from the client selection. 
    % Theorem \ref{th1} shows the effectiveness of \texttt{HCSFed} to reduce selection variance, which also means that the variance upper bound of \texttt{HCSFed} is always lower than any other selections under suitable scaling. 
    % Naturally, \texttt{FedAvg} with \texttt{HCSFed} satisfies the Lemma 5.
    % The dependency between the Lemma and Theorem \ref{th2} indicates that the proof the Lemma 4 is enough to maintain the convergence bound. We defer the proof to Appendix A.3.
    Theorem \ref{th2}~\cite{li2019convergence} uses Lemma \ref{lemma1}-\ref{lemma3} to bound the error of \texttt{FedAvg} with full participation corresponding to the first two terms, while using Lemma \ref{lemma4} to verify the unbiased property and Lemma \ref{lemma5} to bound the variance resulting from the client selection. 
    Theorem \ref{th1} shows the effectiveness of \texttt{HCSFed} to reduce selection variance, which also indicates that the variance of \texttt{HCSFed} is always lower than the mentioned selections unless the population is homogeneous.
    Naturally, \texttt{HCSFed} satisfies the Lemma \ref{lemma5}.
    The dependency between the Lemmas and Theorem \ref{th2} indicates that the proof the Lemma \ref{lemma4} is enough to maintain the convergence bound. We defer the proof to Appendix \ref{app:A.3}.
\end{remark}
We demonstrate the convergence guarantee of \texttt{HCSFed} under $\mu$-strongly convex assumption. As for the non-convex loss function, we refer to the previous proved proposition that \texttt{FedAvg} with any unbiased selections maintains the same FL convergence bound of random selection~(Theorem 2 in \cite{pmlr-v139-fraboni21a}). 

% \comment{
% \begin{corollary}[Rate of Convergence]
% $T_{\epsilon}$ denotes the number of the required round for \texttt{FedAvg} with stratified selection to achieve an $\epsilon$ accuracy, then we have:
% \begin{align}
%     \frac{T_{\epsilon}}{E} &\leq \mathcal{O}\left(\frac{\sum_{k=1}^{N} p_{k}^{2} \sigma_{k}^{2}+\left(\kappa+E+E^2\right) G^{2}}{E}\right) \nonumber \\
%     &+\mathcal{O}\left(\frac{L \Gamma}{E}\right)+\mathcal{O}\left(\frac{E G^{2}}{m}\right).
% \end{align}
% \end{corollary}
% We defer the proof to Appendix B.1.
% \begin{remark}
% As the experimental result shows that the ideal full participant \texttt{FedAvg} achieves fast convergence confirming that the last term which is used to bound the selection variance has a significant impact on the rate of convergence. In section 5.1, we demonstrate the guaranteed capability of our selection scheme on variance reduction. Naturally, we have the proposition below.
% \end{remark}
% \begin{proposition}
%     Although the selection is under a limited sampling ratio, the rate of heterogeneous FL convergence is further accelerated by our sampling scheme.
% \end{proposition}
% }

\section{Experiment} \label{sec5}
% We perform extensive experiments on our selection scheme and compare it with several benchmarks to confirm our theory. 
% %
% Our main observations are: 1) our hybrid selection scheme consistently outperforms all the previous selection schemes~\cite{pmlr-v139-fraboni21a,chen2020optimal} with regard to the convergence rate and stability. 2) the benefit of our selection scheme depends on both the heterogeneity of the clients‘ data and the sampling ratio.
\subsection{Experimental Setup}
We run logistic regression (convex) and a fully connected network with one hidden layer of 50 nodes (non-convex) on MNIST~\cite{lecun1998mnist}.
As for CIFAR-10~\cite{krizhevsky2009learning} and FMNIST~\cite{xiao2017fashion}, we use the same classifier of \texttt{FedAvg}~\cite{mcmahan2017communication:Lixto} composed of 3 convolutional and 2 fully connected layers.
We partition the dataset into 100 clients under both IID and non-IID data distribution.
IID data partition strategy employs the idea of random split to create a uniform federated dataset. 
As for non-IID data partition, we use Dirichlet distribution to partition the entire client set. 
Dirichlet distribution, i.e., $\operatorname{Dir}(\alpha)$, gives to each client the respective partitioning across labels by changing the value of $\alpha$. 
Specifically, the lower the value of $\alpha$, the more heterogeneous the dataset is.
% 
% The schematic table of the Dirichlet distribution is presented in Appendix A.
%
The schematic table of the Dirichlet distribution is presented in Appendix \ref{app:A.1}.
As for the hyperparameters, $N$ is the number of all clients, $q$ is the sampling ratio, $nSGD$ is the times of SGD running locally, $\eta$ is the learning rate, $B$ is the batch size, $T$ is the terminate round. For better comparison, we set all the hyperparameters following the popular FL optimization work \texttt{FedProx}~\cite{li2020federated}.
\begin{table}[bp]
  \vspace{-1em}
  \centering
  \caption{Required rounds for different methods with convex model (logistic regression) to achieve 80$\%$ accuracy on non-IID MNIST and FMNIST. 200$+$ indicates 80 $\%$ accuracy was not reached after 200 rounds.}
  \setlength{\tabcolsep}{0.3mm}
  \scalebox{0.93}{
  \begin{tabular}{llrcrcrc} 
  \toprule
  \multirow{2}*{~} &\multirow{2}*{Methods} &\multicolumn{2}{c}{Sampling Ratio 10\%} &\multicolumn{2}{c}{Sampling Ratio 30\%}
  &\multicolumn{2}{c}{Sampling Ratio 50\%}\\ 
  \cmidrule(r){3-4} \cmidrule(r){5-6}  \cmidrule(r){7-8}
  ~ &~ &\makecell[l]{Num. of Rounds} &\makecell[l]{Speedup}
  &\makecell[l]{Num. of Rounds} &\makecell[l]{Speedup}
  &\makecell[l]{Num. of Rounds} &\makecell[l]{Speedup}
  %   \rowcolor{Gray}
  % 引入makecell包来控制单个单元格的字体格式
  \\  \midrule
%   \multicolumn{1}{>{\columncolor{Gray}}c} 加背景
  \multirowcell{6}{MNIST} &Random & 96 \begin{minipage}[b]{0.12\columnwidth}
		\centering		\raisebox{-.25\height}{\includegraphics[width=\linewidth]{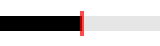}}
	\end{minipage} & \color{gray}(1.0$\times$) 
	& 42 \begin{minipage}[b]{0.12\columnwidth}
		\centering
		\raisebox{-.25\height}{\includegraphics[width=\linewidth]{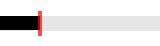}}
	\end{minipage}& \color{gray}(1.0$\times$)  
	& 28 \begin{minipage}[b]{0.12\columnwidth}
		\centering		\raisebox{-.25\height}{\includegraphics[width=\linewidth]{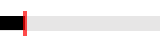}}
	\end{minipage} & \color{gray}(1.0$\times$)
	\\
	&\texttt{SCAFFOLD} & 74 \begin{minipage}[b]{0.12\columnwidth}
		\centering		\raisebox{-.25\height}{\includegraphics[width=\linewidth]{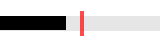}}
	\end{minipage} & \color{gray}(1.3$\times$) 
	& 36 \begin{minipage}[b]{0.12\columnwidth}
		\centering
		\raisebox{-.25\height}{\includegraphics[width=\linewidth]{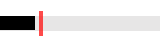}}
	\end{minipage}& \color{gray}(1.2$\times$)  
	& 21 \begin{minipage}[b]{0.12\columnwidth}
		\centering		\raisebox{-.25\height}{\includegraphics[width=\linewidth]{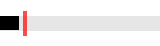}}
	\end{minipage} & \color{gray}(1.3$\times$) 
	\\
	&Importance & 35 \begin{minipage}[b]{0.12\columnwidth}
		\centering		\raisebox{-.25\height}{\includegraphics[width=\linewidth]{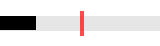}}
	\end{minipage} & \color{gray}(2.7$\times$) 
	& 24 \begin{minipage}[b]{0.12\columnwidth}
		\centering
		\raisebox{-.25\height}{\includegraphics[width=\linewidth]{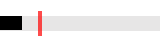}}
	\end{minipage}& \color{gray}(1.8$\times$)  
	& 22 \begin{minipage}[b]{0.12\columnwidth}
		\centering		\raisebox{-.25\height}{\includegraphics[width=\linewidth]{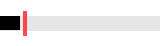}}
	\end{minipage} & \color{gray}(1.3$\times$) 
	\\
	&Cluster & 31 \begin{minipage}[b]{0.12\columnwidth}
		\centering		\raisebox{-.25\height}{\includegraphics[width=\linewidth]{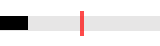}}
	\end{minipage} & \color{gray}(3.1$\times$)
	& 16 \begin{minipage}[b]{0.12\columnwidth}
		\centering
		\raisebox{-.25\height}{\includegraphics[width=\linewidth]{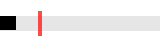}}
	\end{minipage}& \color{gray}(2.6$\times$)  
	& 8 \begin{minipage}[b]{0.12\columnwidth}
		\centering		\raisebox{-.25\height}{\includegraphics[width=\linewidth]{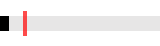}}
	\end{minipage} & \color{gray}(3.5$\times$) 
	\\
	&\textbf{\texttt{HCSFed}} & \color{red}\textbf{9} \begin{minipage}[b]{0.12\columnwidth}
		\centering		\raisebox{-.25\height}{\includegraphics[width=\linewidth]{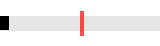}}
	\end{minipage} & \color{gray}(\textbf{10.7}$\times$) 
	& \color{red}\textbf{8} \begin{minipage}[b]{0.12\columnwidth}
		\centering
		\raisebox{-.25\height}{\includegraphics[width=\linewidth]{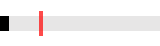}}
	\end{minipage}& \color{gray}(\textbf{5.2}$\times$)  
	& \color{red}\textbf{8} \begin{minipage}[b]{0.12\columnwidth}
		\centering		\raisebox{-.25\height}{\includegraphics[width=\linewidth]{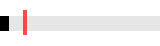}}
	\end{minipage} & \color{gray}(\textbf{3.5}$\times$) 
	\\ \cmidrule{2-8}
	% FMNIST
  \multirowcell{6}{FMNIST} &Random & 200$+$ \begin{minipage}[b]{0.12\columnwidth}
		\centering
		\raisebox{-.25\height}{\includegraphics[width=\linewidth]{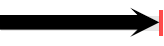}}
	\end{minipage}& \color{gray}(1.0$\times$) & 200$+$ \begin{minipage}[b]{0.12\columnwidth}
		\centering
		\raisebox{-.25\height}{\includegraphics[width=\linewidth]{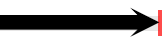}}
	\end{minipage}& \color{gray}(1.0$\times$) 
    & 158 \begin{minipage}[b]{0.12\columnwidth}
		\centering
		\raisebox{-.25\height}{\includegraphics[width=\linewidth]{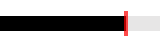}}
	\end{minipage}& \color{gray}(1.0$\times$) \\
    &\texttt{SCAFFOLD} & 180 \begin{minipage}[b]{0.12\columnwidth}
		\centering		\raisebox{-.25\height}{\includegraphics[width=\linewidth]{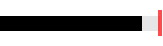}}
	\end{minipage} & \color{gray}(>1.1$\times$) 
	& 144 \begin{minipage}[b]{0.12\columnwidth}
		\centering
		\raisebox{-.25\height}{\includegraphics[width=\linewidth]{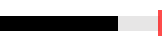}}
	\end{minipage}& \color{gray}(>1.4$\times$)  
	& 115 \begin{minipage}[b]{0.12\columnwidth}
		\centering		\raisebox{-.25\height}{\includegraphics[width=\linewidth]{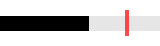}}
	\end{minipage} & \color{gray}(1.4$\times$) 
	\\
	&Importance & 180 \begin{minipage}[b]{0.12\columnwidth}
		\centering		\raisebox{-.25\height}{\includegraphics[width=\linewidth]{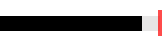}}
	\end{minipage} & \color{gray}(>1.1$\times$)
	& 116 \begin{minipage}[b]{0.12\columnwidth}
		\centering
		\raisebox{-.25\height}{\includegraphics[width=\linewidth]{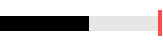}}
	\end{minipage}& \color{gray}(>1.7$\times$)
	& 100 \begin{minipage}[b]{0.12\columnwidth}
		\centering		\raisebox{-.25\height}{\includegraphics[width=\linewidth]{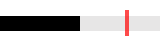}}
	\end{minipage} & \color{gray}(1.6$\times$) 
	\\
	&Cluster & 145 \begin{minipage}[b]{0.12\columnwidth}
		\centering		\raisebox{-.25\height}{\includegraphics[width=\linewidth]{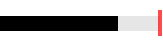}}
	\end{minipage} & \color{gray}(>1.4$\times$)
	& 104 \begin{minipage}[b]{0.12\columnwidth}
		\centering
		\raisebox{-.25\height}{\includegraphics[width=\linewidth]{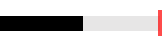}}
	\end{minipage}& \color{gray}(>1.9$\times$) 
	& 75 \begin{minipage}[b]{0.12\columnwidth}
		\centering		\raisebox{-.25\height}{\includegraphics[width=\linewidth]{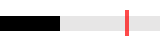}}
	\end{minipage} & \color{gray}(2.1$\times$) 
	\\
	&\textbf{\texttt{HCSFed}} & \color{red}\textbf{81} \begin{minipage}[b]{0.12\columnwidth}
		\centering		\raisebox{-.25\height}{\includegraphics[width=\linewidth]{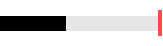}}
	\end{minipage} & \color{gray}(>\textbf{2.5}$\times$) 
	& \color{red}\textbf{80} \begin{minipage}[b]{0.12\columnwidth}
		\centering
		\raisebox{-.25\height}{\includegraphics[width=\linewidth]{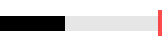}}
	\end{minipage}& \color{gray}(>\textbf{2.5}$\times$) 
	& \color{red}\textbf{75} \begin{minipage}[b]{0.12\columnwidth}
		\centering		\raisebox{-.25\height}{\includegraphics[width=\linewidth]{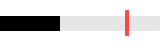}}
	\end{minipage} & \color{gray}(\textbf{2.1}$\times$) 
	\\
   \bottomrule
  \end{tabular}
}
\label{table1}
\end{table}
\begin{figure}[tp]
    % \vspace{-1em}
    \begin{minipage}[t]{0.33\linewidth}
    \centering
    \includegraphics[width=4.33cm]{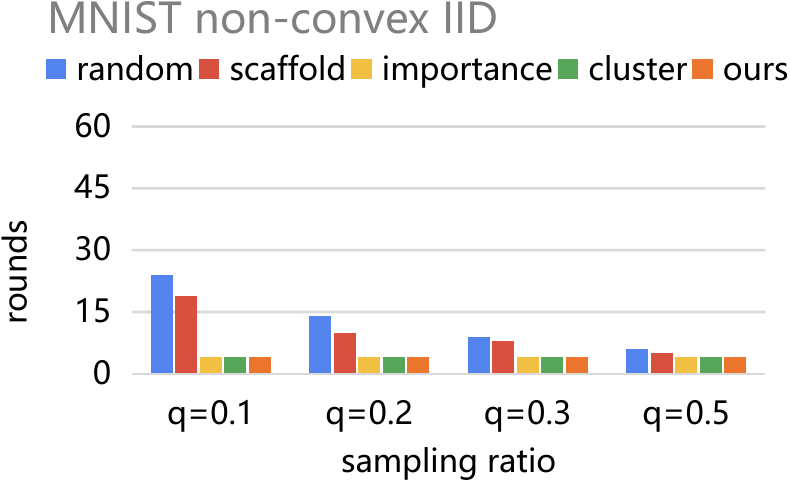}
    \end{minipage}
    \begin{minipage}[t]{0.33\linewidth}
    \centering
    \includegraphics[width=4.33cm]{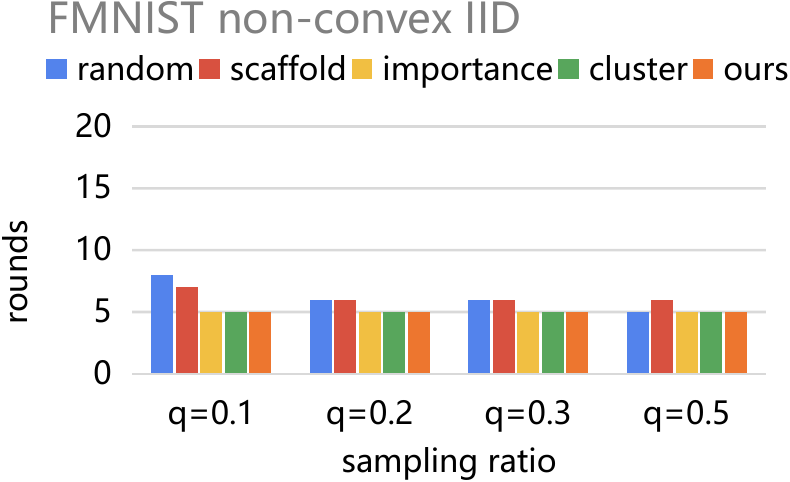}
    \end{minipage}
    \begin{minipage}[t]{0.33\linewidth}
    \centering
    \includegraphics[width=4.33cm]{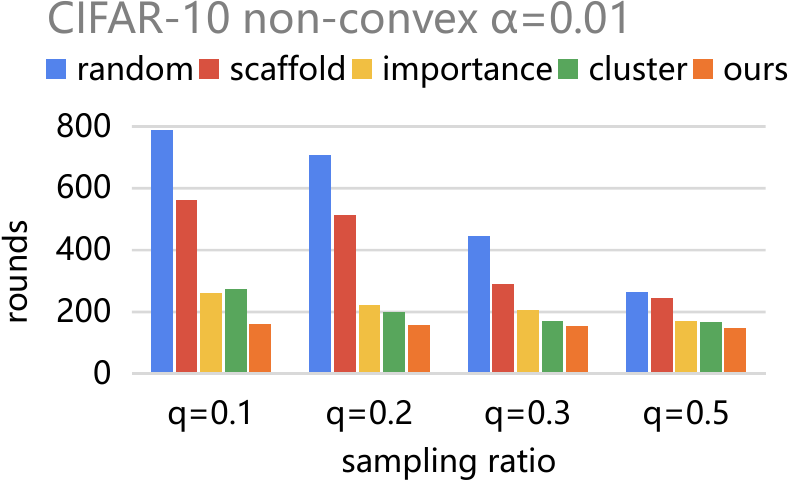}
    \end{minipage}
    \quad
    \begin{minipage}[t]{0.33\linewidth}
    \centering
    \includegraphics[width=4.33cm]{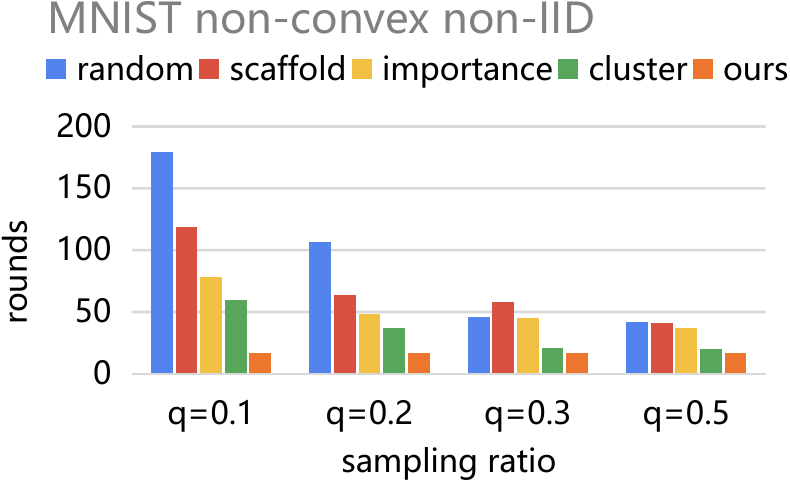}
    \end{minipage}
    \begin{minipage}[t]{0.33\linewidth}
    \centering
    \includegraphics[width=4.33cm]{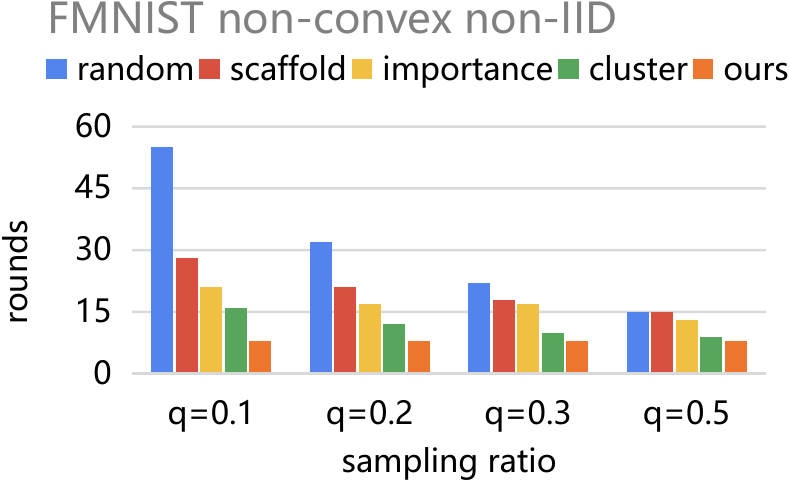}
    \end{minipage}
    \begin{minipage}[t]{0.33\linewidth}
    \centering
    \includegraphics[width=4.33cm]{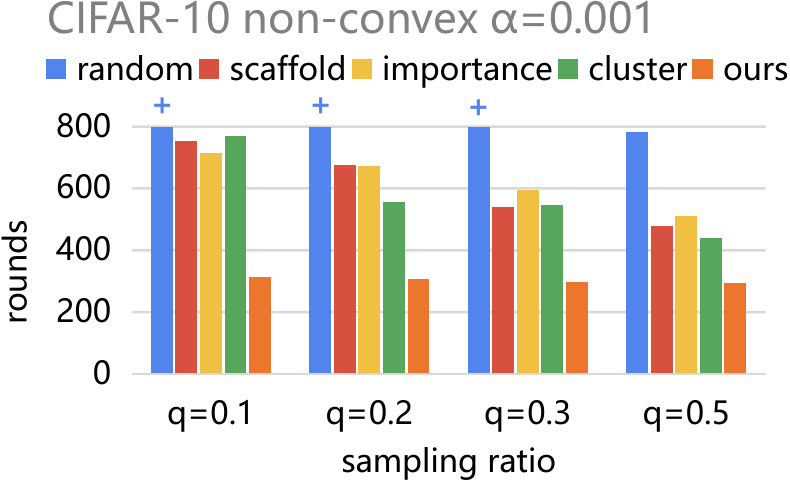}
    \end{minipage}
    \vspace{-1.0em}
    \caption{Required rounds for random, cluster, importance sampling, \texttt{SCAFFOLD}, and \texttt{HCSFed} to achieve 60$\%$ accuracy with $q=0.1$, $N=100$, $nSGD=50$, $\eta=0.01$, $B=50$ on MNIST and FMNIST, with $\alpha \in \{0.01, 0.001\}$, $q=0.1$, $N=100$, $nSGD=80$, $\eta=0.05$, $B=50$ on CIFAR-10.}
    \label{round1}
    \vspace{-1.0em}
\end{figure}

\subsection{Main Results} 
Our evaluation target is to compare the convergence speed of different methods in convex and non-convex FL context.
We choose image classification as our downstream task and compare \texttt{HCSFed} with simple random sampling~\cite{mcmahan2017communication:Lixto}, norm-based importance sampling~\cite{chen2020optimal}, cluster sampling~\cite{pmlr-v139-fraboni21a}, and \texttt{SCAFFOLD}~\cite{karimireddy2020scaffold}.
For better representation of the convergence speed, we follow the previous FL acceleration work setting~\cite{karimireddy2020scaffold} and report the required rounds for different methods with convex model to reach the target test accuracy in Table \ref{table1}.
We can observe that \texttt{HCSFed} consistently achieves the fastest convergence (has the fewest required rounds to reach the target accuracy) across all settings, surpassing not only the client selection schemes but also the SOTA variance-reduced method \texttt{SCAFFOLD}~\cite{karimireddy2020scaffold}. 
We also run \texttt{HCSFed} with non-convex model on different datasets.
The result of non-convex experiment is presented in Figure~\ref{round1}.
\texttt{HCSFed} consistently requires the fewest rounds to reach the target accuracy.
All the results reveal that \texttt{HCSFed} can greatly accelerate the convergence, especially in the low sampling ratio and date heterogeneity case.
As we emphasized in the theoretical analysis, the improvement of \texttt{HCSFed} depends on the sampling ratio and the client heterogeneity. 
The lower the sampling ratio and the more heterogeneous the dataset, the more improvement \texttt{HCSFed} can enjoy. 
Naturally, such great improvement in convergence speed can be attributed to the two key components in \texttt{HCSFed}: 
(i) Different from the previous raw gradient-based methods, we introduce gradient compression to improve the clustering effect by using the more expressive cluster feature representation.
(ii) To further guarantee the robustness of the sampling effect, we introduce a sample size re-allocation scheme and extra importance sampling to control the sampling effect fluctuation in the cluster with low client similarity.
For more details, please refer to Appendix \ref{app:C}.
\subsection{Parameter Sensitivity Study} 
\paragraph{The Number of Clusters~(H).} This experiment demonstrates the sensitivity of \texttt{HCSFed} to the different numbers of clusters. 
We run convex model-based \texttt{HCSFed} on non-IID MNIST with different numbers of clusters, i.e. with $\mathrm{H}\in\left\{5, 6, \cdots, 10\right\}$. 
We report the best classification accuracy and the required rounds for \texttt{HCSFed} with different numbers of clusters to reach the best test accuracy in Figure \ref{Number of Clusters}.
\texttt{HCSFed} achieves stable classification accuracy and similar convergence speeds (measured by the required rounds to achieve the best accuracy), which highlights the robust effect and efficiency of \texttt{HCSFed}.
The results also reveal that the clustering effect of \texttt{HCSFed} is robust regardless of the number of clusters.
Such merit might attribute to the sample size re-allocation scheme or the importance sampling scheme that we use to control the clustering effect fluctuation.
%
% \begin{figure}[htbp]
%     \centering
%     \includegraphics[width=13cm]{Figure/pic_sensitivity.pdf}
%     % \vspace{-1.0em}
%     \caption{Visualization of Sensitivity Study on non-IID MNIST.}
%     \label{Sensitivity}
%     \vspace{-1.0em}
% \end{figure}
\begin{figure}[htbp]
    % \vspace{-1.0em}
    \centering  %居中
    \subfigure[The impact of different numbers of clusters]{
    \centering    %子图居中
    \includegraphics[width=6.75cm]{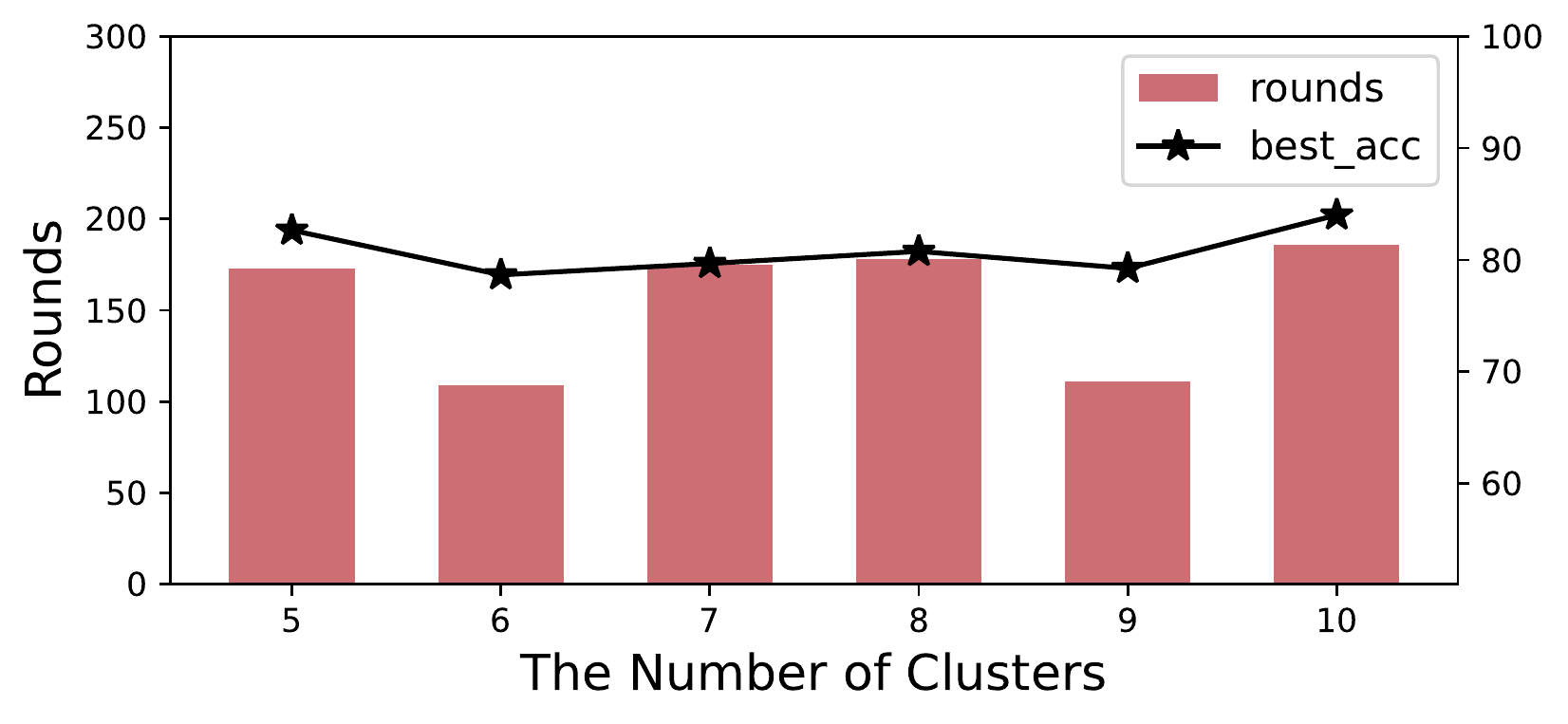} 
    % \caption{Impact of the number of cluster }
    \label{Number of Clusters}
    }
    \subfigure[The impact of different compression rates]{
        \centering    %子图居中
        \includegraphics[width=6.75cm]{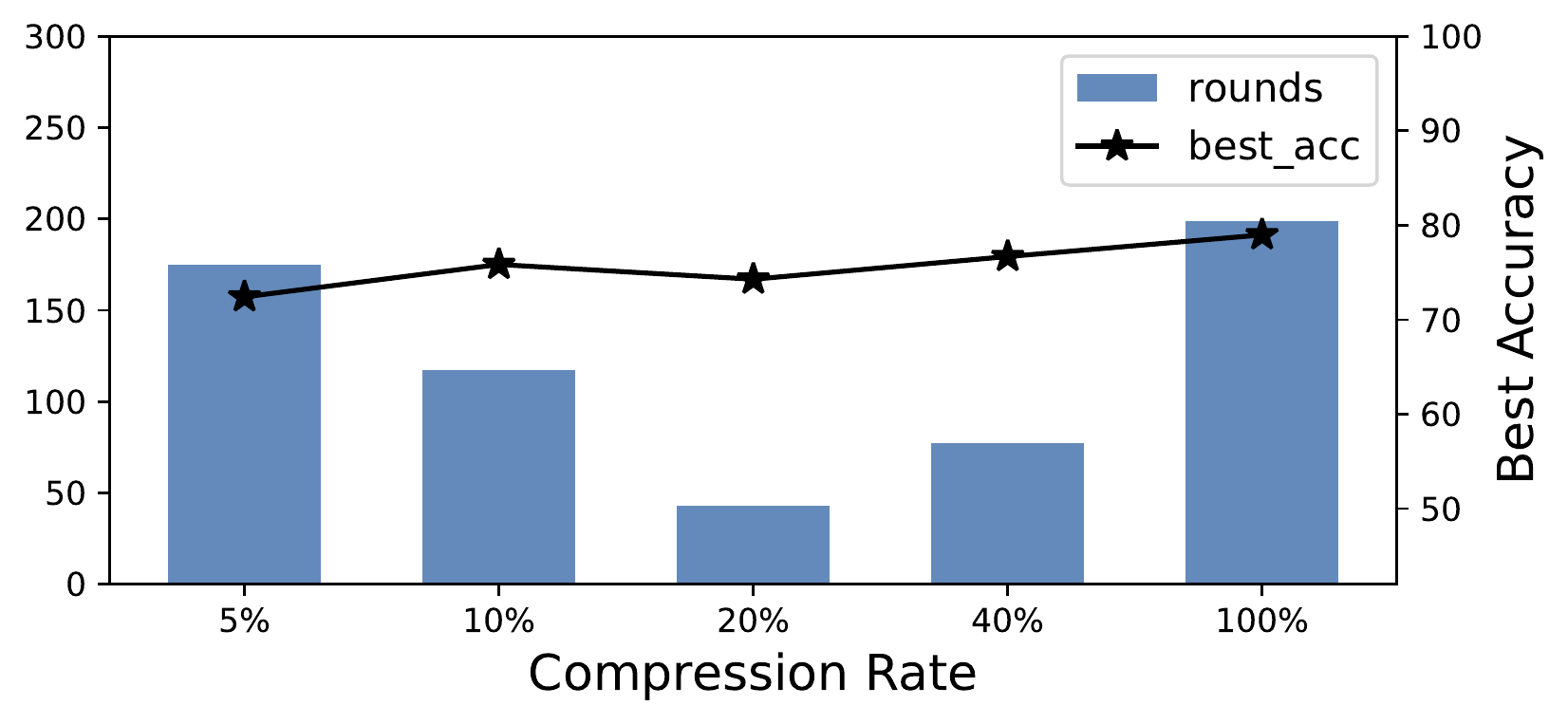}
        %\caption{Impact of the compression rate}
        \label{Compression rate}
    }
    \vspace{-1.0em}
    \caption{Visualization of Sensitivity Study Results on non-IID MNIST.}
    \label{Sensitivity}
    \vspace{-1.0em}
\end{figure}
\paragraph{The Compression Rate~(R).} We also evaluate how different compression rates affect the global model performance. 
We run convex model-based \texttt{HCSFed} on non-IID MNIST with different compression rates (R), i.e. with $\mathrm{R} \in \left\{5\%, 10\%, 20\%, 40\%, 100\% \right\}$. 
We report the best accuracy and the required rounds of \texttt{HCSFed} with different compression rates to reach the best test accuracy in Figure~\ref{Compression rate}. 
The results show that \texttt{HCSFed} achieves different convergence speed as the compression rate changes.
The noticeable information is that both too high and too low compression rates can not achieve good learning efficiency.
Intuitively, the low compression rate can not ensure the integrity of gradient information while the high compression rate is too complicated to learn efficiently.
When we set the compressed rate to $100\%$, i.e., without gradient compression, the results show that both the convergence speed and final accuracy of the model are not satisfactory, which confirms the effectiveness of gradient compression.
In practice, following the previous compression work~\cite{han2015deep}, we set the compression rate with the highest within-group sum of squares to get a good sampling effect.
%
% \begin{figure}[htbp]
%     \centering
%     \includegraphics[width=13cm]{Figure/pic_ablation.pdf}
%     % \vspace{-1.0em}
%     \caption{Visualization of Ablation Study on non-IID MNIST.}
%     \label{Ablation}
%     \vspace{-1.0em}
% \end{figure}
%\
\begin{figure}[htbp]
    % \vspace{-1em}
    \centering  %居中
    \subfigure[{The test accuracy of individual components}]{   %第一张子图
    \centering    %子图居中
    \includegraphics[width=6.6cm]{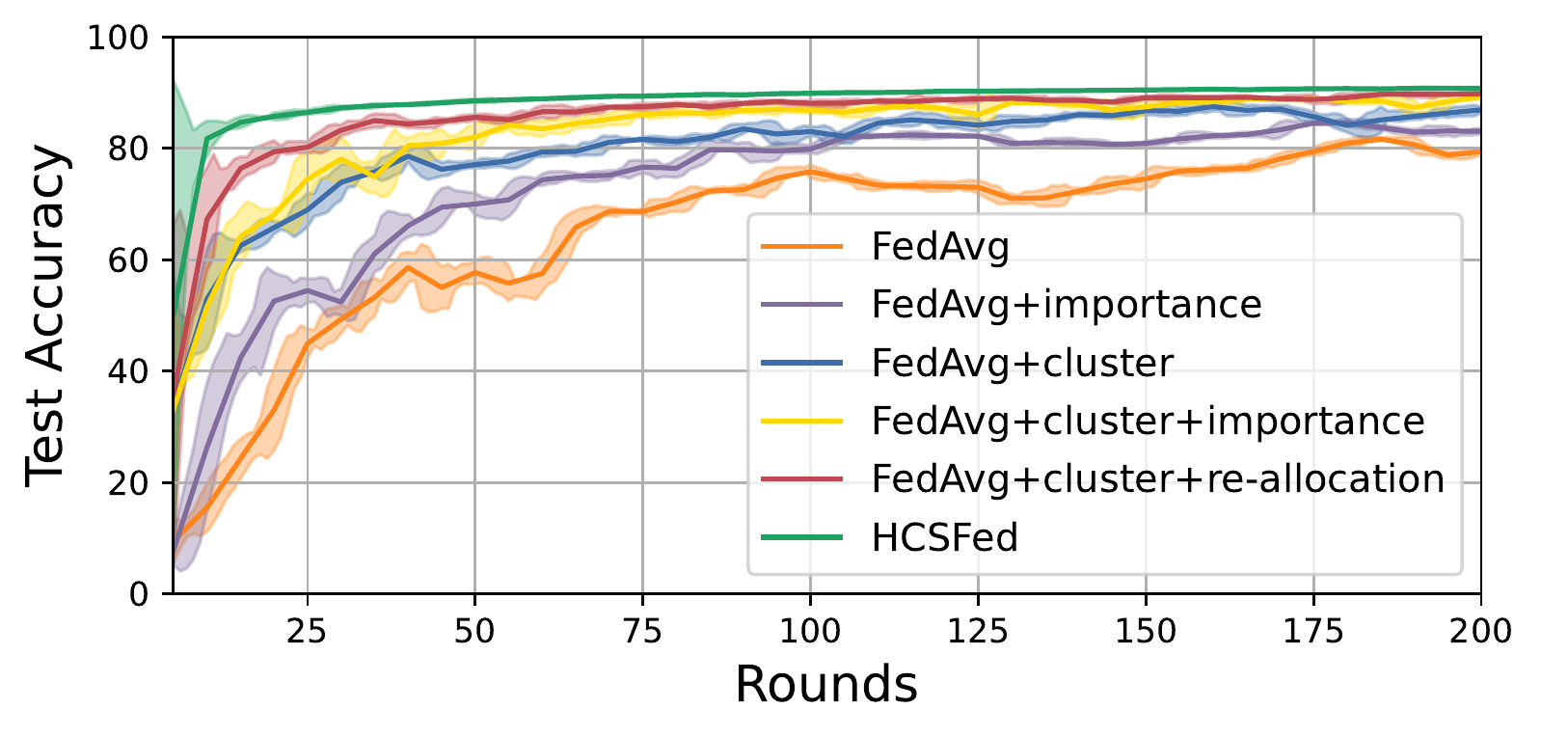} 
    % \caption{The comparison of individual components.}
    \label{Ablation.1}
    }
    \hspace{-0.2cm}
    \subfigure[{The impact of individual components}]{ %第二张子图
    \centering    %子图居中
    \includegraphics[width=7cm]{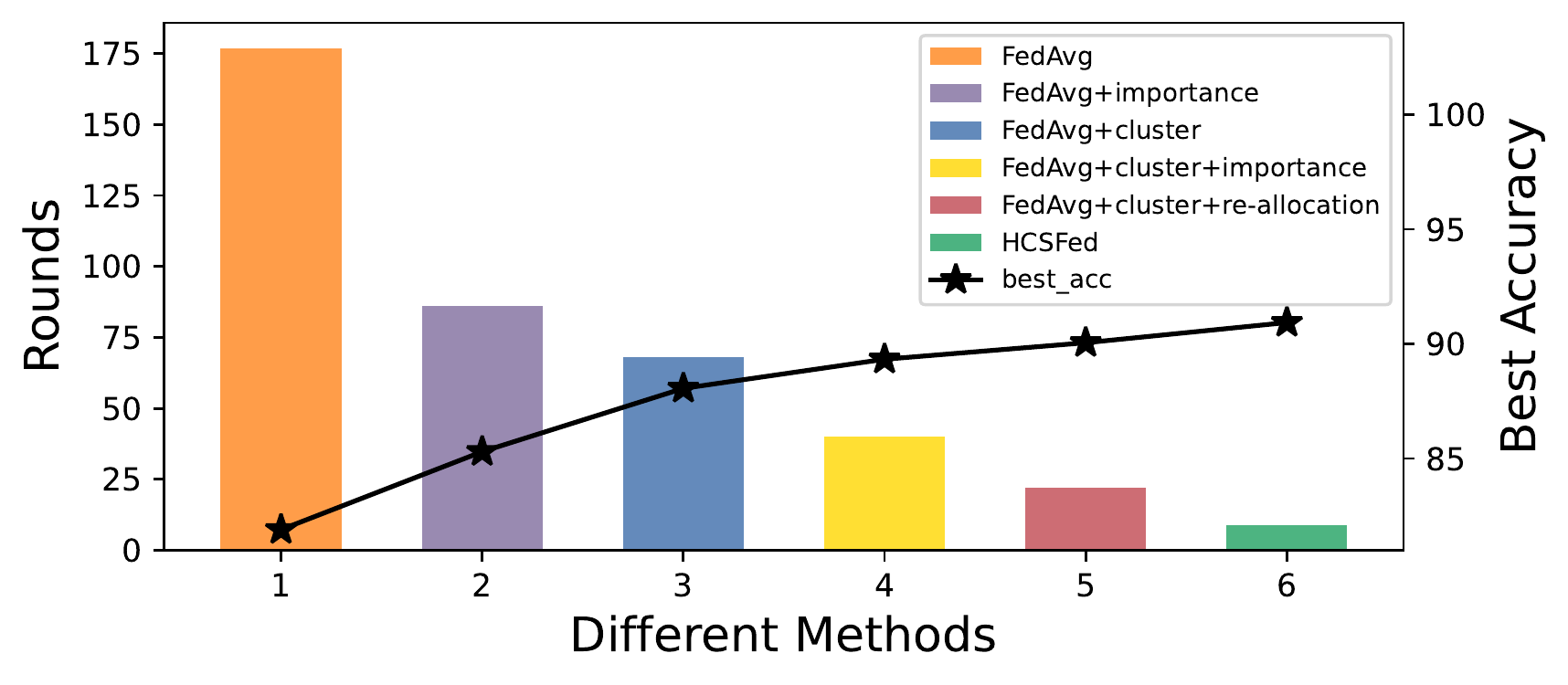}
    \label{Ablation.2}
    % \Caption{The best accuracy and required rounds for different methods to reach best accuracy.}
    }
    \vspace{-1.0em}
    \caption{Visualization of Ablation Study Results on non-IID MNIST.}
    \label{Ablation}
    \vspace{-1.0em}
\end{figure}
\subsection{Ablation Study}
We run logistic regression (convex) on non-IID MNIST with $q = 0.1, nSGD = 50$ to analyse the individual contributions of different components of \texttt{HCSFed} in both final accuracy and convergence speed. 
In Figure~\ref{Ablation.1}, \texttt{HCSFed} is represented as the green line starting from the accuracy of 10$\%$ and ending at 91$\%$, whereas \texttt{FedAvg} is the yellow line starting also from the same position as \texttt{HCSFed} but ending at 81$\%$ after 200 rounds, i.e., 10$\%$ less. 
We report the required rounds for different methods to achieve 80$\%$ accuracy in Figure~\ref{Ablation.2}.
Compared with the plain \texttt{FedAvg}, \texttt{FedAvg} combined with cluster~(blue) or importance sampling~(purple) can accelerate the convergence speed (measured by the required rounds to achieve the target accuracy) by $2.6\times$ and $2.1\times$ times.
The experimental results confirm the important role of the two in improving the learning efficiency.
Impressively, compared with importance sampling, the improvement of the clustering process is more obvious.
This is because clustering can make better use of the correlation among clients, thereby selecting more representative clients to participate in model training.
We also compare \texttt{FedAvg} combined with cluster, \texttt{FedAvg} combined with cluster and re-allocation, \texttt{FedAvg} combined with cluster and importance sampling, and \texttt{HCSFed}.
All the latter methods can achieve a faster convergence in model training than the former, which validates the effectiveness of re-allocation and extra importance sampling.
These results indicate that each component of \texttt{HCSFed} is complementary to each other, and also show the effectiveness of the variance reduction theory in Section~\ref{sec4}.

\section{Concluding Remarks}\label{sec:section6}
In this paper, we propose \texttt{HCSFed}, a novel hybrid clustering-based selection scheme to accelerate the training convergence by variance reduction.
As we mentioned in Section \ref{sec31}, the reduced variance leads to faster convergence.
Therefore, \texttt{HCSFed} can achieve communication cost reduction, as it requires fewer rounds to approach the target training accuracy.
%
% Theoretically, the merits of \texttt{HCSFed} are i) the variance is the same as the random selection in completely homogeneous cases; and ii) if the client datasets are heterogeneous (low client similarity) and the number of clients to be selected is limited, the required rounds of \texttt{HCSFed} to achieve a targeted accuracy are reduced thanks to the variance reduction.
Theoretically, we demonstrate the improvement of the proposed scheme in variance reduction.
We present the convergence guarantee of our proposed approach under the convex assumption.
Experimental results demonstrate the superiority and the effectiveness of our method with both the convex and non-convex models.
In the future, the co-variance among different clients and unavailable client settings will be considered, since it may also lead to performance degradation in FL. In the theoretical aspect, following the latest convergence analysis~\cite{khaled2020tighter}, we plan to loose the widely used but strict assumption “bounded gradient” and convexity for more general analysis of the faster convergence achieved by \texttt{HCSFed}.

\clearpage
% \bibliographystyle{plain}
% \bibliography{neurips_2022.bbl}
% 参考文献

% \section*{References}
% References follow the acknowledgments. Use unnumbered first-level heading for
% the references. Any choice of citation style is acceptable as long as you are
% consistent. It is permissible to reduce the font size to \verb+small+ (9 point)
% when listing the references.
% Note that the Reference section does not count towards the page limit.
% \medskip

% {
% \small

% [1] Alexander, J.A.\ \& Mozer, M.C.\ (1995) Template-based algorithms for
% connectionist rule extraction. In G.\ Tesauro, D.S.\ Touretzky and T.K.\ Leen
% (eds.), {\it Advances in Neural Information Processing Systems 7},
% pp.\ 609--616. Cambridge, MA: MIT Press.

% [2] Bower, J.M.\ \& Beeman, D.\ (1995) {\it The Book of GENESIS: Exploring
%   Realistic Neural Models with the GEneral NEural SImulation System.}  New York:
% TELOS/Springer--Verlag.

% [3] Hasselmo, M.E., Schnell, E.\ \& Barkai, E.\ (1995) Dynamics of learning and
% recall at excitatory recurrent synapses and cholinergic modulation in rat
% hippocampal region CA3. {\it Journal of Neuroscience} {\bf 15}(7):5249-5262.
% }
\comment{
%%%%%%%%%%%%%%%%%%%%%%%%%%%%%%%%%%%%%%%%%%%%%%%%%%%%%%%%%%%%
\clearpage
\section*{Checklist}
\begin{enumerate}

\item For all authors...
\begin{enumerate}
  \item Do the main claims made in the abstract and introduction accurately reflect the paper's contributions and scope?
    \answerYes{}
  \item Did you describe the limitations of your work?
    \answerYes{} See section \ref{sec:section6}, we plan to consider more common settings and loose our assumption.
  \item Did you discuss any potential negative societal impacts of your work?
    \answerNA{} We only work with open source datasets, and we believe that the potential negative impacts associated to it are very low.
  \item Have you read the ethics review guidelines and ensured that your paper conforms to them?
    \answerYes{}
\end{enumerate}

\item If you are including theoretical results...
\begin{enumerate}
  \item Did you state the full set of assumptions of all theoretical results?
    \answerYes{} See section~\ref{sec4}
	\item Did you include complete proofs of all theoretical results?
    \answerYes{} See section~\ref{sec4} and appendix.
\end{enumerate}

\item If you ran experiments...
\begin{enumerate}
    \item Did you include the code, data, and instructions needed to reproduce the main experimental results (either in the supplemental material or as a URL)?
    \answerYes{} Code is made available at \url{https://anonymous.4open.science/r/FL-hybrid-client-selection}
    \item Did you specify all the training details (e.g., data splits, hyperparameters, how they were chosen)?
    \answerYes{} See in section~\ref{sec5}
	\item Did you report error bars (e.g., with respect to the random seed after running experiments multiple times)?
    \answerYes{}
	\item Did you include the total amount of compute and the type of resources used (e.g., type of GPUs, internal cluster, or cloud provider)?
    \answerNA{} Since we only run simulations, this is not applicable.
\end{enumerate}

\item If you are using existing assets (e.g., code, data, models) or curating/releasing new assets...
\begin{enumerate}
  \item If your work uses existing assets, did you cite the creators?
    \answerYes{}
  \item Did you mention the license of the assets?
    \answerYes{} All the data and assets we used in this manuscript are open-source.
  \item Did you include any new assets either in the supplemental material or as a URL?
    \answerYes{} We included an anonymous URL in the supplemental material for the datasets used.
  \item Did you discuss whether and how consent was obtained from people whose data you're using/curating?
    \answerNA{}
  \item Did you discuss whether the data you are using/curating contains personally identifiable information or offensive content?
    \answerNA{}
\end{enumerate}

\item If you used crowdsourcing or conducted research with human subjects...
\begin{enumerate}
  \item Did you include the full text of instructions given to participants and screenshots, if applicable?
    \answerNA{}
  \item Did you describe any potential participant risks, with links to Institutional Review Board (IRB) approvals, if applicable?
    \answerNA{}
  \item Did you include the estimated hourly wage paid to participants and the total amount spent on participant compensation?
    \answerNA{}
\end{enumerate}

\end{enumerate}
}
%%%%%%%%%%%%%%%%%%%%%%%%%%%%%%%%%%%%%%%%%%%%%%%%%%%%%%%%%%%%
\clearpage
% \comment{
\appendix
\section{Extra Illustration} \label{app:A}
\subsection{Dirichlet Distribution Illustration} \label{app:A.1}
Given different values of the concentration parameter $\alpha$ for the Dirichlet distribution, datasets with different degrees of heterogeneity can be generated. 
In particular, higher values of $\alpha$ lead to a more uniform distribution, indicating that each client has an almost equally weighted combination of labels. 
Lower values of $\alpha$ imply weights concentrated more heavily on only one of the labels, or more extreme label membership.
Table \ref{tab2} is an example of Dirichlet distribution used in the experiments. 
As shown in the Table \ref{tab2}, the data distribution on each client is different, and in the case of extreme non-IID, i.e., $\alpha  \rightarrow 0$, most of the data are concentrated under only one label, while the amount of others is almost zero.

\begin{table*}[htbp]
%   \vspace{-1em}
  \centering
  \caption{The actual Dirichlet Distribution (non-IID) generated from CIFAR-10 with $\alpha=0.001$}
  \setlength{\tabcolsep}{1.8mm} %控制表格列距
  \renewcommand\arraystretch{1.1} %控制表格行距
  \scalebox{0.9}{
  \begin{tabular}{c|c|c|c|c|c|c|c|c|c|c|c} 
  \hline
  \multirow{2}*{\textbf{Client ID}} &\multicolumn{10}{c|}{\textbf{Numbers of Samples in the Classes}} &\multirow{2}*{\textbf{Distribution}}\\ 
  \cline{2-11}
  &$c_0$&$c_1$&$c_2$&$c_3$&$c_4$&$c_5$&$c_6$&$c_7$&$c_8$&$c_9$\\ 
  \cline{1-12}
  {k=0}  & 2 & 1 & 33 & 117 & 100 & 6 & 1 & \textbf{0} & 1 & 239 &\begin{minipage}[b]{0.2\columnwidth}
		\centering
		\raisebox{-.3\height}{\includegraphics[width=1\linewidth]{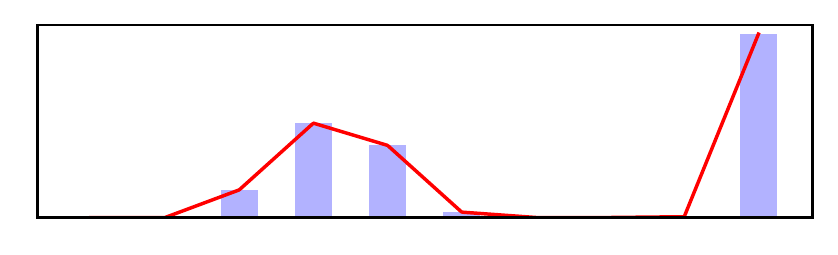}}
	\end{minipage}
  \\
  \cline{1-12}
  {k=1}  & \textbf{0} & \textbf{0} & \textbf{0} & 1 & 1 & 29 & \textbf{467} & \textbf{0} & \textbf{0} & 2
  &\begin{minipage}[b]{0.2\columnwidth}
		\centering
		\raisebox{-.3\height}{\includegraphics[width=1\linewidth]{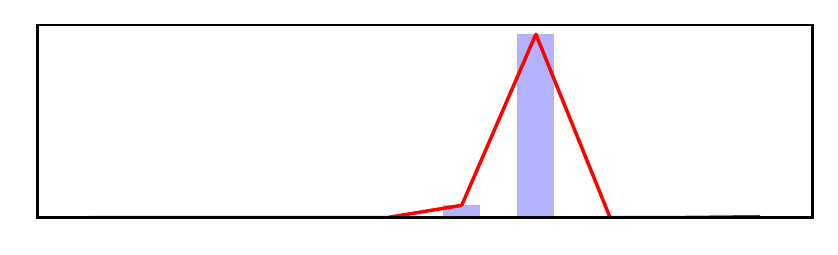}}
	\end{minipage}
  \\ 
  \cline{1-12}
  {k=2}  & 2 & \textbf{397} & 5 & 1 & 2 & 86 & \textbf{0} & 1 & \textbf{0} & 6 
  &\begin{minipage}[b]{0.2\columnwidth}
		\centering
		\raisebox{-.3\height}{\includegraphics[width=1\linewidth]{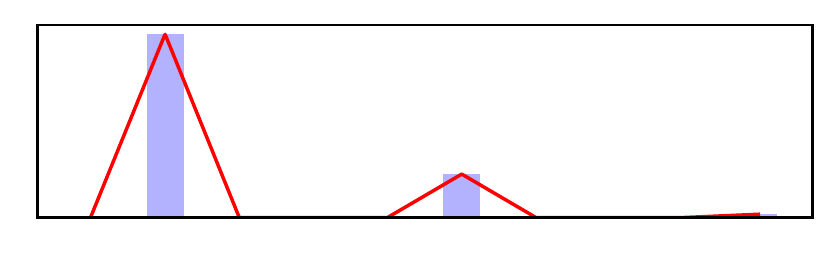}}
	\end{minipage}
  \\ 
  \cline{1-12}
  {k=3}  & 1 & \textbf{0} & \textbf{0} & \textbf{0} & \textbf{0} & 3 & \textbf{0} & \textbf{0} & 125 & \textbf{371} 
  &\begin{minipage}[b]{0.2\columnwidth}
		\centering
		\raisebox{-.3\height}{\includegraphics[width=1\linewidth]{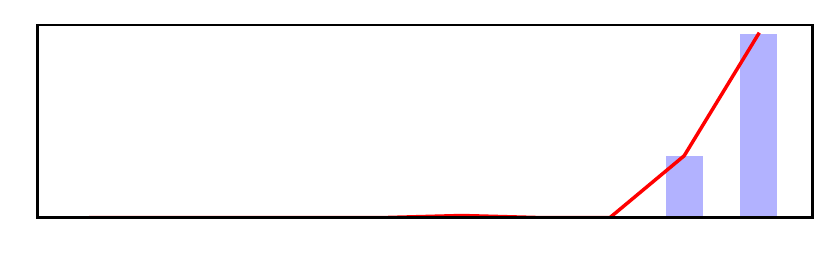}}
	\end{minipage}
  \\ 
  \cline{1-12}
  {k=4}  & 1 & 67 & 5 & \textbf{0} & 15 & \textbf{304} & \textbf{0} & \textbf{0} & 34 & 74 
  &\begin{minipage}[b]{0.2\columnwidth}
		\centering
		\raisebox{-.3\height}{\includegraphics[width=1\linewidth]{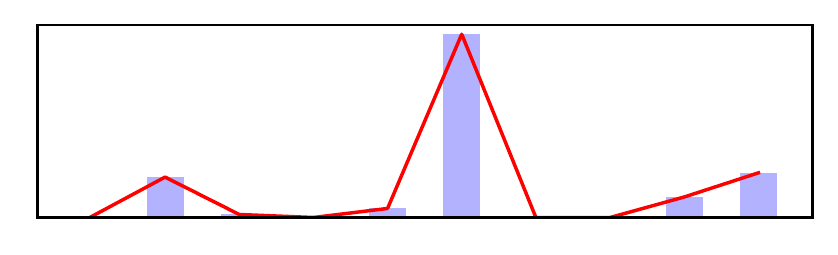}}
	\end{minipage}
  \\ 
  \cline{1-12}
  {$\cdots$}  & $\cdots$ & $\cdots$ &$\cdots$ & $\cdots$ & $\cdots$ & $\cdots$ & $\cdots$ & $\cdots$ & $\cdots$ & $\cdots$ &$\cdots$
  \\
  \cline{1-12}
  {k=95}  & 32 & 213 & \textbf{0} & 94 & 17 & 3 & 138 & \textbf{0} & \textbf{0} & 3
  &\begin{minipage}[b]{0.2\columnwidth}
		\centering
		\raisebox{-.3\height}{\includegraphics[width=1\linewidth]{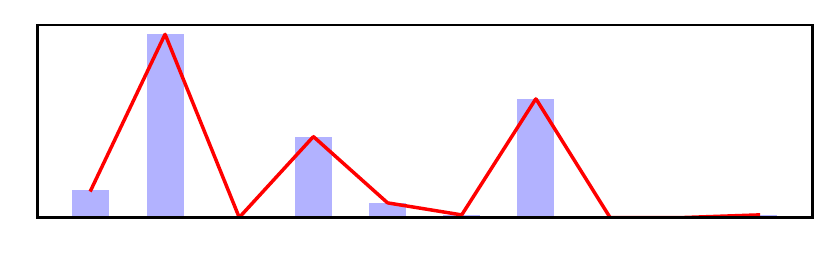}}
	\end{minipage} 
  \\ 
  \cline{1-12}
  {k=96}  & 51 & 36 & 166 & 32 & \textbf{0} & \textbf{0} & 8 & 203 & \textbf{0} & 4 
  &\begin{minipage}[b]{0.2\columnwidth}
		\centering
		\raisebox{-.3\height}{\includegraphics[width=1\linewidth]{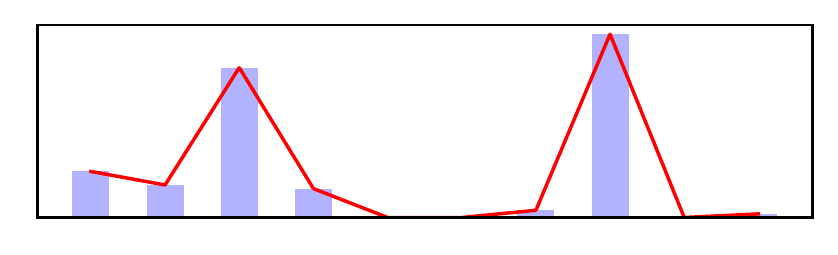}}
	\end{minipage}
  \\ 
  \cline{1-12}
  {k=97}  & 25 & \textbf{0} & \textbf{347} & 17 & 7 & \textbf{0} & \textbf{0} & 44 & \textbf{0} & 60 
  &\begin{minipage}[b]{0.2\columnwidth}
		\centering
		\raisebox{-.3\height}{\includegraphics[width=1\linewidth]{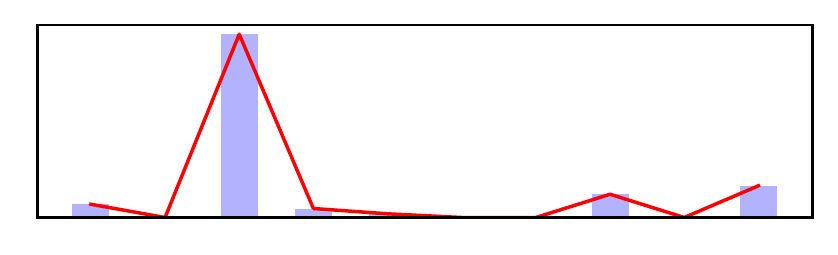}}
	\end{minipage}
  \\ 
  \cline{1-12}
  {k=98}  & \textbf{0} & 1 & \textbf{0} & 2 & 1 & 18 & 3 & 60 & \textbf{413} & 2 
  &\begin{minipage}[b]{0.2\columnwidth}
		\centering
		\raisebox{-.3\height}{\includegraphics[width=1\linewidth]{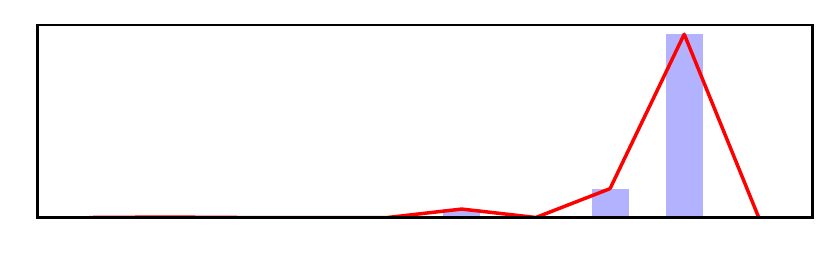}}
	\end{minipage}
  \\ 
  \cline{1-12}
  {k=99}  & \textbf{465} & \textbf{0} & 4 & 2 & 2 & 3 & 4 & 14 & 4 & 2 
  &\begin{minipage}[b]{0.2\columnwidth}
		\centering
		\raisebox{-.3\height}{\includegraphics[width=1\linewidth]{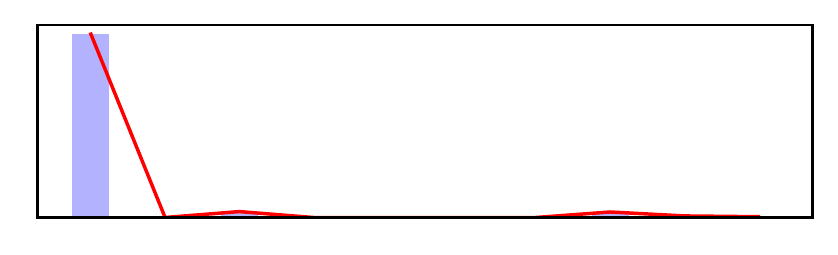}}
	\end{minipage}
  \\
  \hline
  \end{tabular}
  }
  \label{tab2}
\end{table*}

\subsection{Gradient Compress Algorithm} \label{app:A.2}
% \begin{wrapfigure}[20]{R}{0.55\textwidth}
% 	\flushright  % 右对齐环境
% 	 \vspace{-2em}  % 负表示竖直方向向上移动
% \begin{minipage}{1\linewidth}  % 控制算法框大小
\SetInd{0.3em}{0.4em}  % 第一位调整竖线启示位置，第二位控制两条线之间的宽度
\IncMargin{1em} % 使行号向内缩进
\begin{algorithm}[H]
\caption{GC(Gradient Compress)}
\label{alg:algorithm4}
\KwIn{Raw updates in the $t^{th}$ round of the $k^{th}$ client $G_{t}^{k}=\{g_{1}, g_{2}, \ldots, g_{d}\}$}
% \KwIn{Target dimension of the compressed update $d'$ $\left( d' \leq d \right)$}
% \For(\tcp*[f]{target dimension $d^{\prime}$}){$d^{\prime}=\{3,\ldots,\log_{2}d\}$}
% {
\textbf{Initialize} Randomly select $d'$ $g_{i}$ as the group centers \{$x_{1}, x_{2},\ldots, x_{d'}\}$\;
\textbf{Initialize} $C_{j}=\varnothing\ (1 \leq j \leq d^{\prime})$\;
\Repeat{$\forall\ j = \{1,2,\ldots, d^{\prime}\}, x_{j}^{\prime}=x_{j}$}
{\For{each $g_{i},\ i=1,2,\ldots,d$}
{$\lambda_{i} = \arg \min _{j \in \{1,2,\ldots,d'\}} \left\|g_i-x_{j}\right\|_{2}$\; ${C}_{\lambda_{i}}={C}_{\lambda_{i}}\bigcup\left\{ g_i \right\}$\;
}
\For{each cluster $j=1,2,\ldots, d^{\prime}$}
{Calculate new center $x_{j}^{\prime}=\frac{1}{\left|{C}_{j}\right|} \sum_{g_{i} \in {C}_{j}}{g_{i}}$\;
% \eIf{$\boldsymbol{\mu}_{i}^{\prime} \neq \boldsymbol{\mu}_{i}$}
{$x_{j} \leftarrow x_{j}^{\prime}$\;}
% {$\boldsymbol{\mu}_{i}^{\prime} = \boldsymbol{\mu}_{i}$\;}
}
}
% $S_{d^{\prime}} \leftarrow$ compute the within-cluster sum of squares $\sum_{j=1}^{d^{\prime}} \sum_{g_{i} \in c_{j}}\left|g_{i}-c_{j}\right|^{2}$\;
% }
% $d^{\prime} \leftarrow {\arg \min}_{d^{\prime} \in \{3,\ldots,\log_{2}d\}}\{S_{d^{\prime}}\}$\;
\KwOut{$\boldsymbol{X}_{t}^{k}=\left\{x_{1}, x_{2}, \cdots, x_{d'} \right\}$;}
\end{algorithm}

\subsection{Key Lemmas} \label{app:A.3}
Following \cite{li2019convergence}, we present necessary assumptions and extra notations that we used to prove the convergence of FedAvg with random client selection. 
\paragraph{Assumptions.}
The convergence of FedAvg with random sampling scheme has been derived in \cite{li2019convergence}. The proof relies on the assumptions as follows. Assumptions 3 and 4 have been given by  \cite{zhang2013communication,stich2018local,stich2018sparsified}.
\begin{assumption}[L-Smooth] \label{assumption1}
$\forall$ $\mathbf{v}$ and $\mathbf{w}, k=1,\cdots,N F_{k}(\mathbf{v}) \leq F_{k}(\mathbf{w})+(\mathbf{v}-$ $\mathbf{w})^{T} \nabla F_{k}(\mathbf{w})+\frac{L}{2}\|\mathbf{v}-\mathbf{w}\|_{2}^{2}$ where the $\mathbf{v}$,$\mathbf{w}$ are different model parameters.
\end{assumption}
\begin{assumption}[Strongly Convex] \label{assumption2} 
$\forall\ \mathbf{v}$ and $\mathbf{w}, k=1,\cdots,N F_{k}(\mathbf{v}) \geq F_{k}(\mathbf{w})+(\mathbf{v}-$ $\mathbf{w})^{T} \nabla F_{k}(\mathbf{w})+\frac{\mu}{2}\|\mathbf{v}-\mathbf{w}\|_{2}^{2}$ where the $\mathbf{v}$,$\mathbf{w}$ are different model parameters.
\end{assumption}
\begin{assumption}[Bounded Variance] \label{assumption3}
Let $\xi_{t}^{k}$ be sampled from the $k^{th}$ device's local data uniformly at random. The variance of stochastic gradients in each device is bounded:
$$
\mathbb{E} \left\Vert\nabla F_{k}\left(\mathbf{w}_{t}^{k},\xi_{t}^{k}\right)-\nabla F_{k}\left(\mathbf{w}_{t}^{k}\right)\right\Vert^{2} \leq\sigma_{k}^{2},~\forall\ k=1, \cdots, N
$$
\end{assumption}
\begin{assumption}[Bounded Expectation] \label{assumption4}
The expectation of stochastic gradients in squared norm is bounded by $G^2$, i.e., 
$$
\mathbb{E}\left\|\nabla F_{k}\left(\mathbf{w}_{t}^{k}, \xi_{t}^{k}\right)\right\|^{2} \leq G^{2}, \forall\ k=1,\cdots,N,t=1,\cdots T-1
$$
\end{assumption}
\paragraph{Additional Notation} We assume that FedAvg always activates all devices at the beginning of each round and then uses the parameters maintained in only a few sampled devices to produce the next-round parameter. This updating scheme is equivalent to the original. 
Let $\mathcal{I}_{E}$ be the set of global synchronization, i.e., $\mathcal{I}_{E}=\{n E \mid n=1,2, \cdots\}$. If $t+1 \in \mathcal{I}_{E}$, i.e., the time step to communicate.
Then the update of FedAvg with partial devices active can be described as: for all $k \in[N]$,
% \begin{equation}
%     \mathbf{v}_{t+1}^{k} =\mathbf{w}_{t}^{k}-\eta_{t} \nabla F_{k}\left(\mathbf{w}_{t}^{k}, \xi_{t}^{k}\right)
% \end{equation}
% \begin{align}
%     \mathbf{w}_{t+1}^{k} &= \begin{cases}\mathbf{v}_{t+1}^{k} & \text { ,if } t+1 \notin \mathcal{I}_{E} \\ \text { average }\left\{\mathbf{v}_{t+1}^{k}\right\}_{k \in \mathcal{S}_{t+1}} & \text { ,if } t+1 \in \mathcal{I}_{E},\end{cases}
% \end{align}
\begin{gather}
\mathbf{v}_{t+1}^{k}=\mathbf{w}_{t}^{k}-\eta_{t} \nabla F_{k}\left(\mathbf{w}_{t}^{k}, \xi_{t}^{k}\right) \\
\mathbf{w}_{t+1}^{k}= \begin{cases}\mathbf{v}_{t+1}^{k} & \text {,if } t+1 \notin \mathcal{I}_{E} \\
\text { average }\left\{\mathbf{v}_{t+1}^{k}\right\}_{k \in \mathcal{S}_{t+1}} & \text {,if } t+1 \in \mathcal{I}_{E},\end{cases}
\end{gather}
where $\mathcal{S}_{t+1}$ denotes the subset of $(t+1)^{th}$ round. Here, an additional variable $\mathbf{v}_{t+1}^{k}$ is introduced to represent the immediate result of one step SGD update from $\mathbf{w}_{t}^{k}$. We interpret $\mathbf{w}_{t+1}^{k}$ as the parameter obtained after communication steps.
Let $F^{*}$ and $F_{k}^{*}$ be the minimum values of $F$ and $F_{k}$, respectively. We use the term $\Gamma=F^{*}-\sum_{k=1}^{N} p_{k} F_{k}^{*}$ for quantifying the degree of non-IID, where the $p_{k}$ denotes the aggregation weight. If the data are IID, then $\Gamma$ goes to zero as the number of samples grows. If the data are non-IID, then $\Gamma$ is nonzero, and its magnitude reflects the heterogeneity of the data distribution.
\setcounter{counter}{0}
\begin{Lemma}[Results of one step SGD] \label{lemma1}
    Assume Assumption \ref{assumption1} and \ref{assumption2}. If $\eta_{t} \leq \frac{1}{4 L}$, we have
    $$
    \mathbb{E}\left\|\overline{\mathbf{v}}_{t+1}-\mathbf{w}^{\star}\right\|^{2} \leq\left(1-\eta_{t} \mu\right) \mathbb{E}\left\|\overline{\mathbf{w}}_{t}-\mathbf{w}^{\star}\right\|^{2}+\eta_{t}^{2} \mathbb{E}\left\|G_{t}-\overline{G}_{t}\right\|^{2}+6 L \eta_{t}^{2} \Gamma+2 \mathbb{E} \sum_{k=1}^{N} p_{k}\left\|\overline{\mathbf{w}}_{t}-\mathbf{w}_{t}^{k}\right\|^{2}
    $$
where $\Gamma=F^{*}-\sum_{k=1}^{N} p_{k} F_{k}^{\star} \geq 0$, 
\end{Lemma}
\begin{Lemma}[Bounding the variance] \label{lemma2}
    Assume Assumption \ref{assumption3} holds, and $\sigma_{k}$ defined there. It follows that
    $$
    \mathbb{E}\left\|G_{t}-\overline{G}_{t}\right\|^{2} \leq \sum_{k=1}^{N} p_{k}^{2} \sigma_{k}^{2},
    $$
    where the $G_{t}$ is the gradient vector of $t^{th}$ round.
\end{Lemma}
\begin{Lemma}[Bounding the divergence of $\left\{\mathbf{w}_{t}^{k}\right\}$] \label{lemma3}
    Assume Assumption \ref{assumption4}, that $\eta_{t}$ is non-increasing and $\eta_{t} \leq 2 \eta_{t+E}$ for all $t \geq 0$. It follows that
    $$
    \mathbb{E}\left[\sum_{k=1}^{N} p_{k}\left\|\overline{\mathbf{w}}_{t}-\mathbf{w}_{t}^{k}\right\|^{2}\right] \leq 4 \eta_{t}^{2}(E-1)^{2} G^{2} .
    $$
\end{Lemma}
\begin{Lemma}[Unbiased sampling scheme] \label{lemma4}
    If ${\left(t+1\right)}^{th}$ round is the communication round, for our selection with $\mathcal{S}_{t}=\left\{i_{1}, \cdots, i_{m}\right\} \subset[N]$ we have
    $$
    \mathbb{E}\left[\mathbf{w}(\mathcal{S}_t)\right] = \mathbf{w}(\mathcal{K}),
    $$
    where $\mathcal{K}$ denotes the population of clients.
\end{Lemma}
\begin{proof}
    \begin{align}
        \mathbb{E}\left[\mathbf{w}(\mathcal{S}_t)\right]=\mathbb{E}_{\mathcal{S}_{t}} \sum_{k=1}^{m} \mathbf{w}_{i_{k}}=m \mathbb{E}_{\mathcal{S}_{t}} \left[\mathbf{w}_{i_{1}} \right]=m \sum_{k=1}^{N} p_{k} \mathbf{w}_{k}
    \end{align}
\end{proof}
\begin{Lemma}[Bounding the variance of $\mathbf{w}(\mathcal{S}_t)$] \label{lemma5}
    For $t+1 \in \mathcal{I}_{E}$, assume that $\eta_{t}$ is non-increasing and $\eta_{t} \leq 2 \eta_{t+E}$ for all $t \geq 0 .$ We have the following result
    assuming $p_{1}=p_{2}=\cdots=p_{m_{h}}=\frac{1}{N_{h}}$, the expected difference between $\overline{\mathbf{v}}_{t+1}$ and $\overline{\mathbf{w}}_{t+1}$ is bounded by
    $$
    \mathbb{E}_{\mathcal{S}_{t}}\left\|\overline{\mathbf{v}}_{t+1}-\overline{\mathbf{w}}_{t+1}\right\|^{2} \leq \frac{4}{K} \eta_{t}^{2} E^{2} G^{2}
    $$
\end{Lemma}

\section{Proof of Theorem} \label{app:B}

\subsection{Proof of Theorem 1 Variance Reduction}
\paragraph{Additional Notation.}
Divide the population $\mathcal{K}$ consisting of $N$ clients into $\mathrm{H}$ clusters via clustering. 
\begin{itemize}
    \item $N_{h}$ denotes the number of clients in $h^{th}$ cluster, $s.t. \sum_{h=1}^{\mathrm{H}} {N}_{h}={N}$
    \item $m_h$ denotes the number of sampled clients from the $h^{th}$ cluster
    \item $m$ denotes the sample size, $s.t. \sum_{h=1}^{\mathrm{H}} m_h=m$ 
    \item $\mathbf{w}_{h_i}$ denotes the model update $\mathbf{w}$ of the $i^{th}$ client in the $h^{th}$ cluster
    \item $\mathbf{w}_{h} = \sum_{i=1}^{m_h} \frac{\mathbf{w}_{h_i}}{m_h}$ is the sampled averaged model update of the $h^{th}$ cluster
    \item $\overline{\mathbf{w}} = \sum_{h=1}^{\mathrm{H}} \frac{m_h{\mathbf{w}}_h}{m}$ is the overall sampled averaged model update
    \item $\mathbf{W}_{h} = \sum_{i=1}^{N_h} \frac{\mathbf{w}_{h_i}}{N_h}$ is the averaged model update of the $h^{th}$ cluster
    \item $\mathbf{W}(\mathcal{K}) = \sum_{h=1}^{\mathrm{H}} \sum_{i=1}^{N_h} \frac{\mathbf{w}_{h_i}}{N}$ is the averaged model update of entire set $\mathcal{K}$
    \item $\mathbf{w}_{cluster} = \frac{1}{N} \sum_{h=1}^{\mathrm{H}} N_{h} \mathbf{w}_{h}$ is an unbiased estimator of $\mathbf{W}(\mathcal{K})$
    \item $S^{2} = \frac{1}{N}\sum_{i=1}^{N}\left\|\mathbf{w}_{i}-\mathbf{W}(\mathcal{K})\right\|_{2}^2 := \frac{1}{N}\sum_{i=1}^{N}\sigma^{2}$
    \item $S_{h}{}^{2} = \sum_{i=1}^{N_h} \frac{\left\|\mathbf{w}_{h_i}-\mathbf{W}_{h}\right\|_{2}^2}{N_h-1}$
    \item $s_{h}{}^{2} = \sum_{i=1}^{m_h} \frac{\left\|\mathbf{w}_{h_i} - \mathbf{w}_h\right\|_{2}^2}{m_h-1}$
    \item ${Q}_h = \frac{N_h}{N}$ is the proportion of clients in the $h^{th}$ cluster
    \item ${q}_h = \frac{m_h}{m}$ is the proportion of sampled clients in the $h^{th}$ cluster
\end{itemize}
\begin{proof}[Proof of Theorem 1]
    \textbf{Derive the Variance of Random Selection.}
    Assuming that each observation has variance $\sigma^2$, then we get
    \begin{align}
    \mathbb{V}(\mathbf{w}_{rand}) &=\mathbb{E}\left\|\overline{\mathbf{w}}-\mathbf{W}(\mathcal{K})\right\|_{2}^2\\ 
    &=\frac{1}{m^{2}}\mathbb{E}\left\| \sum_{i=1}^{m}[\mathbf{w}_{i}-\mathbf{W}(\mathcal{K})]\right\|_{2}^2 \\
    &=\underbrace{\frac{1}{m^{2}}\mathbb{E} \left[\sum_{i=1}^{m}\left\|\mathbf{w}_{i}-\mathbf{W}(\mathcal{K})\right\|_{2}^{2}\right]}_{Quadratic~Term}\\
    &+\underbrace{\frac{1}{m^{2}}\mathbb{E} \left[\sum_{i}^{m} \sum_{\neq j}^{m}\left[\mathbf{w}_{i}-\mathbf{W}(\mathcal{K})\right]^T\left[\mathbf{w}_{j}-\mathbf{W}(\mathcal{K})\right]\right]}_{ Cross-Product~Term} \\
    &=\frac{1}{m^{2}} \sum_{i=1}^{m} \mathbb{E}\left\|\mathbf{w}_{i}-\mathbf{W}(\mathcal{K})\right\|_{2}^2\\
    &+\frac{1}{m^{2}} \underbrace{\sum_{i}^{m} \sum_{\neq j}^{m} \mathbb{E}\left[\left[\mathbf{w}_{i}-\mathbf{W}(\mathcal{K})\right]^T\left[\mathbf{w}_{j}-\mathbf{W}(\mathcal{K})\right]\right]}_{Setting~to~K}\\
    &=\frac{1}{m^{2}} \sum_{i=1}^{m} \sigma^{2}+\frac{K}{m^{2}} \\
    &=\frac{N-1}{Nm} S^{2}+\frac{K}{m^{2}},
\end{align}
    
where we set $K=\sum_{i}^{m} \sum_{\neq j}^{m} \mathbb{E}\left[\left[\mathbf{w}_{i}-\mathbf{W}(\mathcal{K})\right]^{T}\left[\mathbf{w}_{j}-\mathbf{W}(\mathcal{K})\right]\right]$ for convenience.

\textbf{Find the Expression of $K$.}
    In order to find $K$, we consider,
    \begin{align}
        \mathbb{E} \left[\left[\mathbf{w}_{i}-\mathbf{W}(\mathcal{K})\right]\left[\mathbf{w}_{j}-\mathbf{W}(\mathcal{K})\right]\right]
        =\frac{1}{N(N-1)} \sum_{k}^{N} \sum_{\neq \ell}^{N} \left[\left[\mathbf{w}_{k}-\mathbf{W}(\mathcal{K})\right]^T\left[(\mathbf{w}_{l}-\mathbf{W}(\mathcal{K})\right]\right].
    \end{align}
    Meanwhile, we have,
    \begin{align}
        \sum_{k=1}^{N}\left[\mathbf{w}_{k}-\mathbf{W}(\mathcal{K})\right] &=  \sum_{k=1}^{N}\mathbf{w}_{k}-N \mathbf{W}(\mathcal{K}) \\
        &= N \mathbf{W}(\mathcal{K})-N \mathbf{W}(\mathcal{K}) = 0,
    \end{align}
    i.e.,
    \begin{align}
         \left\|\sum_{k=1}^{N}[\mathbf{w}_{k}-\mathbf{W}(\mathcal{K})]\right\|_{2}^{2} = 0.
    \end{align}
    And the left can be constructed as, 
    \begin{align}
        \left\|\sum_{k=1}^{N}[\mathbf{w}_{k}-\mathbf{W}(\mathcal{K})]\right\|_{2}^{2} &=\sum_{k=1}^{N}\left\|\mathbf{w}_{k}-\mathbf{W}(\mathcal{K})\right\|_{2}^{2} \\
        &+\sum_{k}^{N} \sum_{\neq \ell}^{N}\left[\left[\mathbf{w}_{k}-\mathbf{W}(\mathcal{K})\right]^T\left[\mathbf{w}_{\ell}-\mathbf{W}(\mathcal{K})\right]\right].
    \end{align}
    Simplify it, we will get
    \begin{align}
        0 &=(N-1) S^{2}+\sum_{k}^{N} \sum_{\neq \ell}^{N}\left[\left[\mathbf{w}_{k}-\mathbf{W}(\mathcal{K})\right]^T\left[\mathbf{w}_{\ell}-\mathbf{W}(\mathcal{K})\right]\right],
    \end{align}
    equal to, 
     \begin{align}
        \sum_{k}^{N} \sum_{\neq \ell}^{N}\left[\left[\mathbf{w}_{k}-\mathbf{W}(\mathcal{K})\right]^T\left[\mathbf{w}_{\ell}-\mathbf{W}(\mathcal{K})\right]\right] = -(N-1) S^{2}.
    \end{align}
    Therefore,
    \begin{align}
        \frac{1}{N(N-1)} \sum_{k}^{N} &\sum_{\neq \ell}^{N}\left[\left[\mathbf{w}_{k}-\mathbf{W}(\mathcal{K})\right]^T\left[\mathbf{w}_{\ell}-\mathbf{W}(\mathcal{K})\right]\right] \\
        &=\frac{1}{N(N-1)}\left[-(N-1) S^{2}\right] \\
        &=-\frac{S^{2}}{N},
    \end{align}
    thus 
    \begin{align}
        K&=\sum_{i}^{m} \sum_{\neq j}^{m} \mathbb{E}\left[\left[\mathbf{w}_{i}-\mathbf{W}(\mathcal{K})\right]^T\left[\mathbf{w}_{j}-\mathbf{W}(\mathcal{K})\right]\right]\\
        &= m(m-1) \frac{1}{N(N-1)} \sum_{k}^{N} \sum_{\neq \ell}^{N}\left[\left[\mathbf{w}_{k}-\mathbf{W}(\mathcal{K})\right]^T\left[(\mathbf{w}_{l}-\mathbf{W}(\mathcal{K})\right]\right]\\
        &= -m(m-1)\frac{S^2}{N},
    \end{align}
    and substitute the value of $K$, the variance of $\mathbf{w}_{rand}$ is 
    \begin{align}
    \mathbb{V}\left(\mathbf{w}_{rand}\right) &=\frac{N-1}{N m} S^{2}-\frac{1}{n^{2}} m(m-1) \frac{S^{2}}{N} \\
    &=\frac{N-m}{N m} S^{2}.
    \end{align}
   If $N$ is infinite (large enough), we can get
    \begin{align}
        \mathbb{V}\left(\mathbf{w}_{rand}\right) &=\frac{N-m}{N m} S^{2} \\
        &= (\frac{1}{m}-\frac{1}{N})S^2 \cong \frac{S^2}{m}.
    \end{align}
\textbf{Derive the Variance of Plain Clustering Selection.} As prior work constructed, clustering selection is always applied under plain proportional allocation, where the number of sampled clients $m_{h}$ from the $h^{th}$ cluster is proportional to its cluster size $N_{h}$, i.e., $m_{h}=m\frac{N_{h}}{N}$. And we have 
\begin{align}
    \mathbb{V}(\mathbf{w}_{cluster})
    =\sum_{h=1}^{\mathrm{H}} Q_{h}{}^{2} \mathbb{V}\left(\mathbf{w}_{h}\right)+\sum_{h(\neq j)=1}^{\mathrm{H}} \sum_{j=1}^{m_{h}} Q_{h} Q_{j} \operatorname{Cov}\left(\mathbf{w}_{h}, \mathbf{w}_{j}\right).
\end{align}
For the former we have
\begin{align}
     \mathbb{V}\left(\mathbf{w}_{h}\right)&=\frac{N_{h}-m_{h}}{N_{h} m_{h}} S_{h}{}^{2},
\end{align}
and for the latter (covariance) we have
\begin{align}
    \operatorname{Cov}\left(\mathbf{w}_{h}, \mathbf{w}_{j}\right)&=0, h \neq j,
\end{align}
where
\begin{align}
    S_{h}{}^{2}&=\frac{1}{N_{h}-1} \sum_{j=1}^{N_{h}}\left\|\mathbf{w}_{h_{ j}}-\mathbf{W}_{h}\right\|_{2}^{2},
\end{align}
thus
\begin{align}
    \mathbb{V}(\mathbf{w}_{cluster})&=\sum_{h=1}^{\mathrm{H}}\left(\frac{N_{h}-m_{h}}{N_{h} m_{h}}\right) {Q}_{h}{}^{2} S_{h}{}^{2}.
\end{align}
Therefore we can get
\begin{align}
    \mathbb{V}(\mathbf{w}_{cluster})
    &=\sum_{h=1}^{\mathrm{H}}\left(\frac{N_{h}-\frac{m}{N} N_{h}}{N_{h}\frac{m}{N} N_{h}}\right)\left(\frac{N_{h}}{N}\right)^{2} S_{h}{}^{2} \\
    &= \frac{N-m}{N m} \sum_{h=1}^{\mathrm{H}} \frac{N_{h} S_{h}{}^{2}}{N} \\
    &= \frac{N-m}{N m} \sum_{h=1}^{\mathrm{H}} {Q}_{h} S_{h}{}^{2} \\
    &\cong \frac{\sum_{h=1}^{\mathrm{H}}N_{h}S_{h}{}^{2}}{mN}.
\end{align}

\textbf{Derive the Variance of Clustering Selection with Sample Size Re-allocation.} 
We apply clustering selection under sample size re-allocation, where the number of sampled clients $m_{h}$ from the $h^{th}$ cluster is proportional to both cluster's size $N_{h}$ and the variability of cluster measured by $S_{h}$, i.e.,
\begin{align}
m_{h}=\frac{N_{h}S_{h}}{\sum_{h=1}^{\mathrm{H}} N_{h} S_{h}} \cdot m.
\end{align}
We can get
\begin{align}
\mathbb{V}\left(\mathbf{w}_{cludiv}\right) &=\sum_{h=1}^{\mathrm{H}}\left(\frac{1}{m_{h}}-\frac{1}{N_{h}}\right) Q_{h}{}^{2} S_{h}{}^{2} \\
&=\sum_{h=1}^{\mathrm{H}} \frac{Q_{h}{}^{2} S_{h}{}^{2}}{m_{h}}-\sum_{h=1}^{\mathrm{H}} \frac{Q_{h}{}^{2} S_{h}{}^{2}}{N_{h}} \\
&=\sum_{h=1}^{\mathrm{H}}\left[Q_{h}{}^{2} S_{h}{}^{2}\left(\frac{\sum_{h=1}^{\mathrm{H}} N_{h} S_{h}}{m N_{h} S_{h}}\right)\right]-\sum_{h=1}^{\mathrm{H}} \frac{Q_{h}{}^{2} S_{h}{}^{2}}{N_{h}} \\
&=\sum_{h=1}^{\mathrm{H}}\left[\frac{1}{m} \cdot \frac{N_{h} S_{h}}{N^{2}}\left(\sum_{h=1}^{\mathrm{H}} N_{h} S_{h}\right)\right]-\sum_{h=1}^{\mathrm{H}} \frac{Q_{h}{}^{2} S_{h}{}^{2}}{N_{h}} \\
&=\frac{1}{m}\left(\sum_{h=1}^{\mathrm{H}} \frac{N_{h} S_{h}}{N}\right)^{2}-\sum_{h=1}^{\mathrm{H}} \frac{Q_{h}{}^{2} S_{h}{}^{2}}{N_{h}}\\
&=\frac{1}{m}\left(\sum_{h=1}^{\mathrm{H}} Q_{h} S_{h}\right)^{2}-\frac{1}{N} \sum_{h=1}^{\mathrm{H}} Q_{h} S_{h}{}^{2}\\
&=\frac{1}{N^2}\frac{\left(\sum_{h=1}^{\mathrm{H}}N_{h}S_{h}\right)^2}{m} - \frac{1}{N^2}\sum_{h=1}^{\mathrm{H}}{N_{h}S_{h}{}^2}\\
&\cong \frac{1}{mN^2}\left(\sum_{h=1}^{\mathrm{H}}N_{h}S_{h}\right)^2.
\end{align}
Based on all the above, We have these equations below when approximations are used,
\begin{align}
\mathbb{V}\left(\mathbf{w}_{rand}\right) &= \frac{N-m}{Nm}S^2 \\
&\cong \frac{S^2}{m},\\
\mathbb{V}\left(\mathbf{w}_{cluster}\right) &= \frac{N-m}{N}\cdot\frac{\sum_{h=1}^{\mathrm{H}}N_{h}S_{h}{}^{2}}{mN} \\
&\cong \frac{\sum_{h=1}^{\mathrm{H}}N_{h}S_{h}{}^{2}}{mN},\\
\mathbb{V}\left(\mathbf{w}_{cludiv}\right) &= \frac{1}{N^2}\cdot\frac{\left(\sum_{h=1}^{\mathrm{H}}N_{h}S_{h}\right)^2}{m} - \frac{1}{N^2}\sum_{h=1}^{\mathrm{H}}{N_{h}S_{h}{}^2}\\
&\cong \frac{1}{mN^2}\left(\sum_{h=1}^{\mathrm{H}}N_{h}S_{h}\right)^2.
\end{align}

\textbf{Relationship.}
In order to compare $\mathbb{V}({w}_{rand})$ and $\mathbb{V}({w}_{cluster})$, we first attempt to express $S^{2}$ as a function of $S_{h}{}^{2}$.
    \begin{align}
        \left(N-1\right)S^2 &= \sum_{h=1}^{\mathrm{H}}\sum_{i=1}^{m_h}\left\|\mathbf{w}_{h_i}-\mathbf{W}(\mathcal{K})\right\|_{2}^2     \\
        &= \sum_{h=1}^{\mathrm{H}}\sum_{i=1}^{m_h}\left\|\mathbf{w}_{h_i}-\mathbf{W}_h\right\|_{2}^2\\
        &+ \sum_{h=1}^{\mathrm{H}}N_{h}\left\|\mathbf{W}_h-\mathbf{W}(\mathcal{K})\right\|_{2}^2\\
        &= \sum_{h=1}^{\mathrm{H}}\left(N_{h}-1\right)S_{h}{}^{2} + \sum_{h=1}^{\mathrm{H}}N_{h}\left\|\mathbf{W}_h -\mathbf{W}(\mathcal{K})\right\|_{2}^2
    \end{align}
    \begin{align}
        \frac{N-1}{N} S^{2}=\sum_{h=1}^{\mathrm{H}} \frac{N_{h}-1}{N} S_{h}{}^{2}+\sum_{h=1}^{\mathrm{H}} \frac{N_{h}}{N}\left\|\mathbf{W}_{h}-\mathbf{W}(\mathcal{K})\right\|_{2}^{2}
    \end{align}
We assume that $N_{h}$ is large enough to permit the approximation for simplification
    \begin{align}
        \frac{N_{h}-1}{N_{h}} \approx 1 \text { and } \frac{N-1}{N} \approx 1
    \end{align}
Thus
    \begin{align}
    {S}^{2}=\sum_{h=1}^{\mathrm{H}} \frac{{N}_{h}}{N} {S}_{h}{}^{2}+\sum_{h=1} ^{\mathrm{H}}\frac{{N}_{h}}{N}\left\|\mathbf{W}_{h}-\mathbf{W}(\mathcal{K})\right\|_{2}^{2}
    \end{align}
Therefore
    \begin{align}
    \mathbb{V}\left(\mathbf{w}_{rand}\right)=\frac{S^{2}}{m}&=\frac{\sum_{h=1}^{\mathrm{H}} N_{h} S_{h}{ }^{2}}{mN}\\
    &+\frac{\sum_{h=1}^{\mathrm{H}} N_{h}\left\|\mathbf{W}_{h}-\mathbf{W}(\mathcal{K})\right\|_{2}^{2}}{mN}\\
    &=\mathbb{V}\left(\mathbf{w}_{cluster}\right)+\frac{\sum_{h=1}^{\mathrm{H}} N_{h}\left\|\mathbf{W}_{h}-\mathbf{W}(\mathcal{K})\right\|_{2}^{2}}{mN}
    \end{align}
which shows that
    \begin{equation}
    \mathbb{V}\left(\mathbf{w}_{cluster}\right) \leq \mathbb{V}\left(\mathbf{w}_{rand}\right)
    \end{equation}
Unless $\mathbf{W}_h=\mathbf{W}(\mathcal{K})$ for every h, we must have $\mathbb{V}\left(\mathbf{w}_{cludiv}\right) \leq \mathbb{V}\left(\mathbf{w}_{cluster}\right)$.\\
The difference is 
    \begin{align}
\mathbb{V}\left(\mathbf{w}_{cluster}\right)&=\mathbb{V}\left(\mathbf{w}_{cludiv}\right)+\frac{1}{mN}\sum_{h=1}^{\mathrm{H}}N_{h}\left(S_{h}-\overline{S}\right)^2.
    \end{align}
This shows that 
    \begin{equation}
    \mathbb{V}\left(\mathbf{w}_{cludiv}\right) \leq \mathbb{V}\left(\mathbf{w}_{cluster}\right),
    \end{equation}
unless $S_h=\overline{S}$ for every h, i.e., the clusters have equal variability. Therefore, we get 
    \begin{equation}
    \mathbb{V}\left(\mathbf{w}_{cludiv}\right) \leq \mathbb{V}\left(\mathbf{w}_{cluster}\right)\leq
\mathbb{V}\left(\mathbf{w}_{rand}\right)
    \end{equation}
In this paper, we proposed to apply the importance selection based on the norm of the gradient to each cluster instead of random selection. Here we present the variance-reduction relationship between random selection and importance selection. Please note that this part is directly adapted from prior importance sampling work~\cite{katharopoulos2018not}, which is not our contribution.
\begin{align}
    \operatorname{Tr}\left(\mathbb{V}_{rand}\left[G_{i}\right]\right)&-\operatorname{Tr}\left(\mathbb{V}_{import}\left[c_{i} G_{i}\right]\right) \\
&=\mathbb{E}_{rand}\left[\left\|G_{i}\right\|_{2}^{2}\right]-\mathbb{E}_{import}\left[c_{i}^{2}\left\|G_{i}\right\|_{2}^{2}\right]
\end{align}
Using the fact that $c_{i}=\frac{1}{N_{h} G_{i}}$, $I_{i} = \frac{\left\|G_{i}\right\|_{2}}{\sum^{N_h}_{i=1}\left\|G_i\right\|_{2}}$, $u=\frac{1}{N_{h}}$, we have
\begin{equation}
    \mathbb{E}_{import}\left[c_{i}^{2}\left\|G_{i}\right\|_{2}^{2}\right]=\left(\frac{1}{N_{h}} \sum_{i=1}^{N_{h}}\left\|G_{i}\right\|_{2}\right)^{2}
\end{equation}
Then simplify it, we can get 
\begin{align}
    &\operatorname{Tr}\left(\mathbb{V}_{rand}\left[G_{i}\right]\right)-\operatorname{Tr}\left(\mathbb{V}_{import}\left[w_{i} G_{i}\right]\right) \\
&=\frac{1}{N_{h}} \sum_{i=1}^{N_{h}}\left\|G_{i}\right\|_{2}^{2}-\left(\frac{1}{N_{h}} \sum_{i=1}^{N_{h}}\left\|G_{i}\right\|_{2}\right)^{2} \\
&=\frac{\left(\sum_{i=1}^{N_{h}}\left\|G_{i}\right\|_{2}\right)^{2}}{N_{h}^{3}} \sum_{i=1}^{N_{h}}\left(N_{h}^{2} \frac{\left\|G_{i}\right\|_{2}^{2}}{\left(\sum_{i=1}^{N_{h}}\left\|G_{i}\right\|_{2}\right)^{2}}-1\right) \\
&=\frac{\left(\sum_{i=1}^{N_{h}}\left\|G_{i}\right\|_{2}\right)^{2}}{N_{h}} \sum_{i=1}^{N_{h}}\left(I_{i}^{2}-u^{2}\right)
\end{align}
Using the fact that$\sum_{i=1}^{N_{h}} u=1$,  we can complete the derivation.
\begin{align}
    &\operatorname{Tr}\left(\mathbb{V}_{rand}\left[G_{i}\right]\right)-\operatorname{Tr}\left(\mathbb{V}_{import}\left[w_{i} G_{i}\right]\right) \\
&=\frac{\left(\sum_{i=1}^{N_{h}}\left\|G_{i}\right\|_{2}\right)^{2}}{N_{h}} \sum_{i=1}^{N_{h}}\left(I_{i}-u\right)^{2} \\
&=\left(\frac{1}{N_{h}} \sum_{i=1}^{}\left\|G_{i}\right\|_{2}\right)^{2} N_{h} \|I-u\|_{2}^{2} 
\end{align}
\end{proof}

\clearpage

\section{Additional Experiments} \label{app:C}
\subsection{Influence of the Sampling Ratio}
As the sampling ratio $q$ decreases, the performance of all schemes becomes unacceptable except ours.
\begin{table}[htbp]
  \vspace{-1.0em}  % 调整表格与上文的距离
  \centering
  \setlength{\tabcolsep}{4mm} % 列间距
  \caption{Final test accuracy of multiple FL algorithms  with different sampling schemes under convex model on MNIST, FMNIST, setting parameters $q \in \{0.1, 0.2, 0.3, 0.5\}$, $N=100$, $nSGD=50$, $\eta=0.05$, $B=50$.}
\scalebox{0.98}{
\begin{tabular}
  {llcccc} \toprule
  \multirow{2}*{} &\multirow{2}*{Methods} &\multicolumn{2}{c}{MNIST} &\multicolumn{2}{c}{FMNIST} \\
  \cmidrule(r){3-4} \cmidrule(r){5-6}
% \rowcolor{Gray} 为每一行添加背景色
  ~ &~
  &IID&\cellcolor{white}non-IID
  &IID&\cellcolor{white}non-IID  
  \\ \midrule
  \multirow{5}*{q=0.1} &{Random} & 86.8 \color{gray}±0.0 & 79.3 \color{gray}±0.6 & 75.9 \color{gray}±0.0 & 62.6 \color{gray}±1.0 \\ %\cmidrule{2-5}
  &\texttt{SCAFFOLD} & 84.2 \color{gray}±0.2 & 78.8 \color{gray}±1.3 & 73.5 \color{gray}±0.0 & 70.7 \color{gray}±0.0 \\
  &{Importance}  & 90.9 \color{gray}±0.0 & 87.2 \color{gray}±1.2 & 82.5 \color{gray}±0.1 & 73.2 \color{gray}±1.6 \\
  &{Cluster}  & 90.91 \color{gray}±0.0 & 88.0 \color{gray}±0.4 & 82.6 \color{gray}±0.1 & 74.3 \color{gray}±2.0 \\
  &{\textbf{\texttt{HCSFed}}}  & 90.9 \color{gray}±0.0 & \textbf{89.0 \color{gray}±0.0} & 82.5 \color{gray}±0.1 & \textbf{78.8 \color{gray}±0.0} \\ \midrule
  \multirow{5}*{q=0.2} &{Random}& 88.4 \color{gray}±0.0 & 83.2 \color{gray}±0.3 & 78.6 \color{gray}±0.1 & 71.6 \color{gray}±1.1 \\ %\cmidrule{2-5}
  &\texttt{SCAFFOLD} & 85.2 \color{gray}±0.0 & 86.6 \color{gray}±0.2 & 74.2 \color{gray}±0.0 & 71.0 \color{gray}±0.0 \\
  &{Importance}  & 90.8 \color{gray}±0.0 & 87.5 \color{gray}±1.0 & 82.6 \color{gray}±0.1 & 75.1 \color{gray}±2.0 \\
  &{Cluster}  & 90.89 \color{gray}±0.0 & 88.6 \color{gray}±0.3 & 82.5 \color{gray}±0.1 & 77.3 \color{gray}±1.1 \\
  &{\textbf{\texttt{HCSFed}}}  & 91.0 \color{gray}±0.0 & \textbf{89.1 \color{gray}±0.1} & \textbf{82.5 \color{gray}±0.1} & \textbf{78.9 \color{gray}±0.1} \\ \midrule
  \multirow{5}*{q=0.3} &{Random}& 89.3 \color{gray}±0.0 & 86.1 \color{gray}±0.2 & 80.1 \color{gray}±0.1 & 71.3 \color{gray}±0.2 \\ %\cmidrule{2-5}
  &\texttt{SCAFFOLD} & 84.9 \color{gray}±0.1 & 87.0 \color{gray}±0.1 & 74.4 \color{gray}±0.0 & 71.1 \color{gray}±0.0 \\
  &{Importance}  & 90.8 \color{gray}±0.0 & 88.4 \color{gray}±0.4 & 82.6 \color{gray}±0.0 & 75.4 \color{gray}±1.9 \\
  &{Cluster}  & 90.85 \color{gray}±0.0 & 89.1 \color{gray}±0.1 & 82.5 \color{gray}±0.0 & 77.4 \color{gray}±0.5 \\
  &{\textbf{\texttt{HCSFed}}}  & 90.9 \color{gray}±0.0 & \textbf{89.0 \color{gray}±0.0} & 82.6 \color{gray}±0.1 & \textbf{78.9 \color{gray}±0.0} \\ \midrule
  \multirow{5}*{q=0.5} &{Random}& 90.0 \color{gray}±0.0 & 87.2 \color{gray}±0.1 & 81.2 \color{gray}±0.0 & 75.5 \color{gray}±0.3 \\ %\cmidrule{2-5}
  &\texttt{SCAFFOLD} & 85.0 \color{gray}±0.1 & 87.2 \color{gray}±0.1 & 74.5 \color{gray}±0.0 & 70.9 \color{gray}±0.0 \\
  &{Importance}  & 90.9 \color{gray}±0.0 & 89.0 \color{gray}±0.2 & 82.5 \color{gray}±0.0 & 77.2 \color{gray}±1.0 \\
  &{Cluster}  & 90.87 \color{gray}±0.0 & 89.0 \color{gray}±0.1 & 82.5 \color{gray}±0.0 & 78.7 \color{gray}±0.4 \\
  &{\textbf{\texttt{HCSFed}}}  & 90.9 \color{gray}±0.0 & \textbf{89.0 \color{gray}±0.0} & 82.5 \color{gray}±0.0 & \textbf{78.9 \color{gray}±0.0} \\ 
  \bottomrule
  \end{tabular}}
%   \vspace{-8em}
\label{tab3}
\end{table}

\begin{table}[htbp]
  \vspace{-1.0em}  % 调整表格与上文的距离
  \centering
  \setlength{\tabcolsep}{1mm} % 列间距
  \caption{Final test accuracy of multiple FL algorithms with different sampling schemes under non-convex model on MNIST, FMNIST and CIFAR-10, setting parameters $q \in \{0.1, 0.2, 0.3, 0.5\}$, $N=100$, $nSGD=50$ for MNIST and FMNIST, $nSGD=80$ for CIFAR-10, $\eta=0.05$, $B=50$.}
\scalebox{0.98}{
\begin{tabular}
  {llccccccc} \toprule
  \multirow{2}*{} &\multirow{2}*{Methods} &\multicolumn{2}{c}{MNIST} &\multicolumn{2}{c}{FMNIST} &\multicolumn{3}{c}{CIFAR-10}\\
  \cmidrule(r){3-4} \cmidrule(r){5-6} \cmidrule(r){7-9}
  ~ &~
  &IID&\cellcolor{white}non-IID
  &IID&\cellcolor{white}non-IID  
  & \cellcolor{white}IID &$\alpha=0.01$&$\alpha=0.001$ \\ \midrule
  \multirow{5}*{q=0.1} &{Random} & 87.0 \color{gray}±0.0 & 59.7 \color{gray}±1.4 & 87.6 \color{gray}±0.1 & 76.9 \color{gray}±0.2 & 40.3 \color{gray}±0.2 & 25.8 \color{gray}±0.4 & 20.5 \color{gray}±0.5\\ %\cmidrule{2-5}
%   &\texttt{SCAFFOLD} & 79.7 \color{gray}±0.2 & 72.9 \color{gray}±0.5 & 67.7 \color{gray}±0.1 & 64.6 \color{gray}±0.7 & 24.4 \color{gray}±1.1 & 40.5 \color{gray}±1.0 & 32.8 \color{gray}±0.8\\
  &{Importance}  & 92.9 \color{gray}±0.1 & 71.2 \color{gray}±9.5 & 90.4 \color{gray}±0.1 & 85.1 \color{gray}±2.2 & 66.0 \color{gray}±0.2 & 39.5 \color{gray}±3.3 & 24.9 \color{gray}±2.6\\
  &{Cluster}  & 92.9 \color{gray}±0.0 & 73.6 \color{gray}±3.7 & 90.6 \color{gray}±0.2 & 88.9 \color{gray}±4.6 & 65.5 \color{gray}±0.3 & 37.2 \color{gray}±3.9 & 30.0 \color{gray}±4.7\\
  &{\textbf{\texttt{HCSFed}}}  & 92.9 \color{gray}±0.0 & \textbf{83.3 \color{gray}±0.0} & 90.5 \color{gray}±0.1 & \textbf{92.0 \color{gray}±0.2} & 65.7 \color{gray}±0.3 & \textbf{41.2 \color{gray}±1.8} & \textbf{38.8 \color{gray}±0.6} \\ \midrule
  \multirow{5}*{q=0.2} &{Random} & 89.3 \color{gray}±0.0 & 70.8 \color{gray}±1.6 & 88.8 \color{gray}±0.0 & 80.5 \color{gray}±0.5 & 49.8 \color{gray}±0.4 & 29.7 \color{gray}±0.3 & 25.1 \color{gray}±0.2\\ %\cmidrule{2-5}
%   &\texttt{SCAFFOLD} & 81.3 \color{gray}±0.1 & 75.2 \color{gray}±0.3 & 67.5 \color{gray}±0.1 & 73.4 \color{gray}±0.1 & 12.7 \color{gray}±3.1 & 36.7 \color{gray}±1.1 & 30.1 \color{gray}±2.1\\
  &{Importance}  & 92.9 \color{gray}±0.0 & 72.9 \color{gray}±9.0 & 90.3 \color{gray}±0.0 & 90.8 \color{gray}±0.6 & 65.7 \color{gray}±0.2 & 40.7 \color{gray}±2.9 & 30.5 \color{gray}±2.0\\
  &{Cluster}  & 92.9 \color{gray}±0.0 & 80.3 \color{gray}±1.7 & 90.4 \color{gray}±0.1 & 90.7 \color{gray}±1.0 & 65.6 \color{gray}±0.3 & 42.0 \color{gray}±1.4 & 33.0 \color{gray}±3.6\\
  &{\textbf{\texttt{HCSFed}}}  & 92.8 \color{gray}±0.0 & \textbf{83.5 \color{gray}±0.1} & 90.4 \color{gray}±0.1 & \textbf{92.1 \color{gray}±0.2} & 65.4 \color{gray}±0.3 & \textbf{41.6 \color{gray}±0.9} & \textbf{39.2 \color{gray}±2.0} \\ \midrule
  \multirow{5}*{q=0.3} &{Random} & 90.1 \color{gray}±0.0 & 73.4 \color{gray}±2.1 & 89.4 \color{gray}±0.0 & 83.7 \color{gray}±0.3 & 54.8 \color{gray}±0.4 & 34.9 \color{gray}±0.6 & 26.6 \color{gray}±0.5\\ %\cmidrule{2-5}
%   &\texttt{SCAFFOLD} & 81.3 \color{gray}±0.0 & 76.2 \color{gray}±0.0 & 67.8 \color{gray}±0.2 & 73.6 \color{gray}±0.3 & 24.6 \color{gray}±1.5 & 21.1 \color{gray}±2.6 & 31.7 \color{gray}±1.3\\
  &{Importance}  & 92.9 \color{gray}±0.0 & 76.3 \color{gray}±3.5 & 90.6 \color{gray}±0.1 & 90.5 \color{gray}±1.3 & 65.7 \color{gray}±0.2 & 42.0 \color{gray}±2.3 & 32.6 \color{gray}±3.4\\
  &{Cluster}  & 92.9 \color{gray}±0.0 & 81.8 \color{gray}±1.5 & 90.5 \color{gray}±0.1 & 91.6 \color{gray}±0.5 & 66.1 \color{gray}±0.3 & 43.0 \color{gray}±1.0 & 34.6 \color{gray}±2.1\\
  &{\textbf{\texttt{HCSFed}}}  & 92.8 \color{gray}±0.0 & \textbf{83.3 \color{gray}±0.1} & 90.2 \color{gray}±0.1 & \textbf{92.2 \color{gray}±0.1} & 65.4 \color{gray}±0.3 & \textbf{42.3 \color{gray}±0.6} & \textbf{39.7 \color{gray}±0.7} \\ \midrule
  \multirow{5}*{q=0.5} &{Random} & 91.4 \color{gray}±0.0 & 80.9 \color{gray}±1.3 & 90.1 \color{gray}±0.1 & 87.9 \color{gray}±0.3 & 61.1 \color{gray}±0.2 & 39.4 \color{gray}±0.4 & 30.5 \color{gray}±0.6\\ %\cmidrule{2-5}
%   &\texttt{SCAFFOLD} & 81.6 \color{gray}±0.0 & 78.7 \color{gray}±0.0 & 68.0 \color{gray}±0.3 & 71.3 \color{gray}±0.8 & 16.2 \color{gray}±1.3 & 28.6 \color{gray}±1.1 & 28.8 \color{gray}±0.7\\
  &{Importance}  & 92.9 \color{gray}±0.0 & 80.8 \color{gray}±3.8 & 90.6 \color{gray}±0.1 & 90.5 \color{gray}±1.5 & 65.8 \color{gray}±0.3 & 43.2 \color{gray}±1.2 & 35.4 \color{gray}±1.7\\
  &{Cluster}  & 92.9 \color{gray}±0.0 & 83.5 \color{gray}±0.2 & 90.4 \color{gray}±0.1 & 92.0 \color{gray}±0.2 & 65.9 \color{gray}±0.4 & 44.0 \color{gray}±0.6 & 36.3 \color{gray}±1.0\\
  &{\textbf{\texttt{HCSFed}}}  & 92.9 \color{gray}±0.0 & \textbf{83.3 \color{gray}±0.1} & 90.4 \color{gray}±0.1 & \textbf{92.0 \color{gray}±0.2} & 65.7 \color{gray}±0.5 & \textbf{42.4 \color{gray}±0.7} & \textbf{39.8 \color{gray}±0.6} \\ 
  \bottomrule
  \end{tabular}}
\label{tab4}
\end{table}
\begin{figure}[!htbp]
    \centering
    \includegraphics[width=13.5cm]{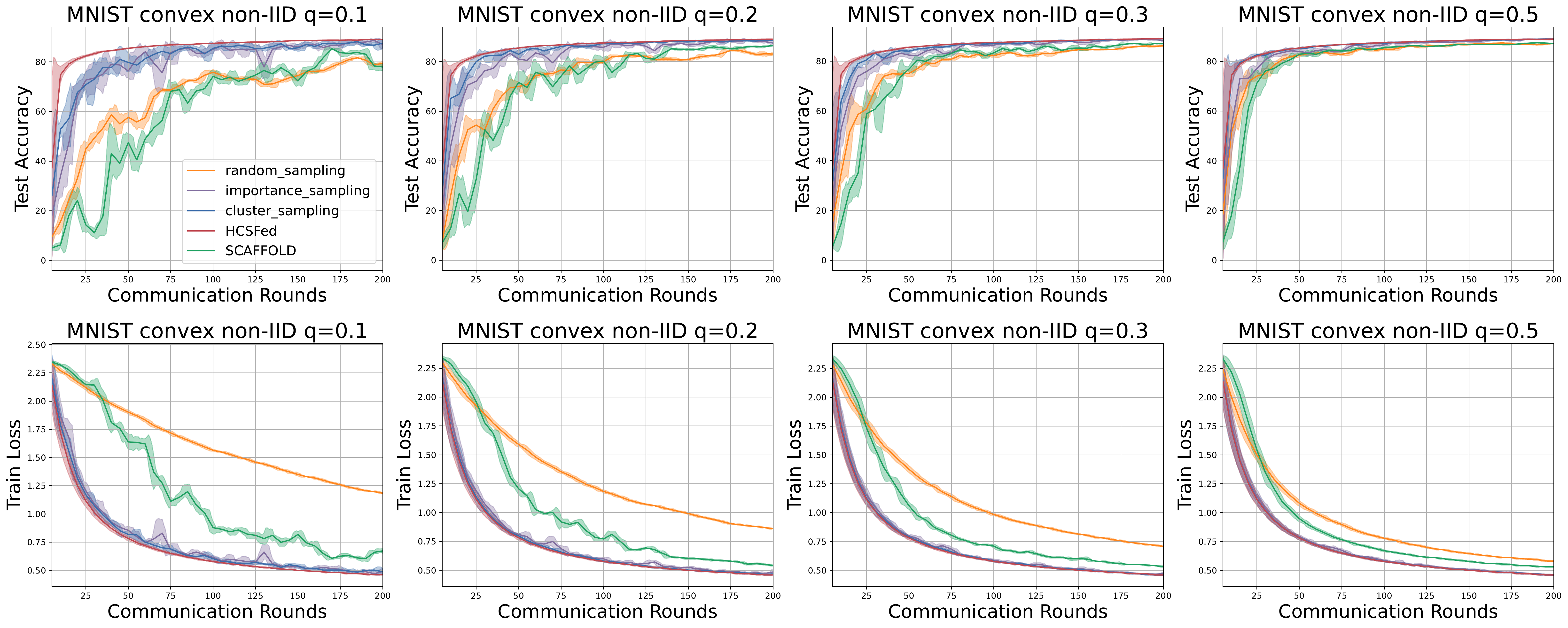}
    \caption{Impact of sampling ratio $q$ on the performance with convex model. We compare \texttt{HCSFed} with simple random sampling, importance sampling, cluster sampling, SCAFFOLD on MNIST under non-IID, setting parameters $q\in \{0.1, 0.2, 0.3, 0.5\}$, $N=100$, $nSGD=50$, $\eta=0.01$, $B=50$.}
    \label{app:fig1}
\end{figure}
\begin{figure}[!htbp]
    \centering
    \includegraphics[width=13.5cm]{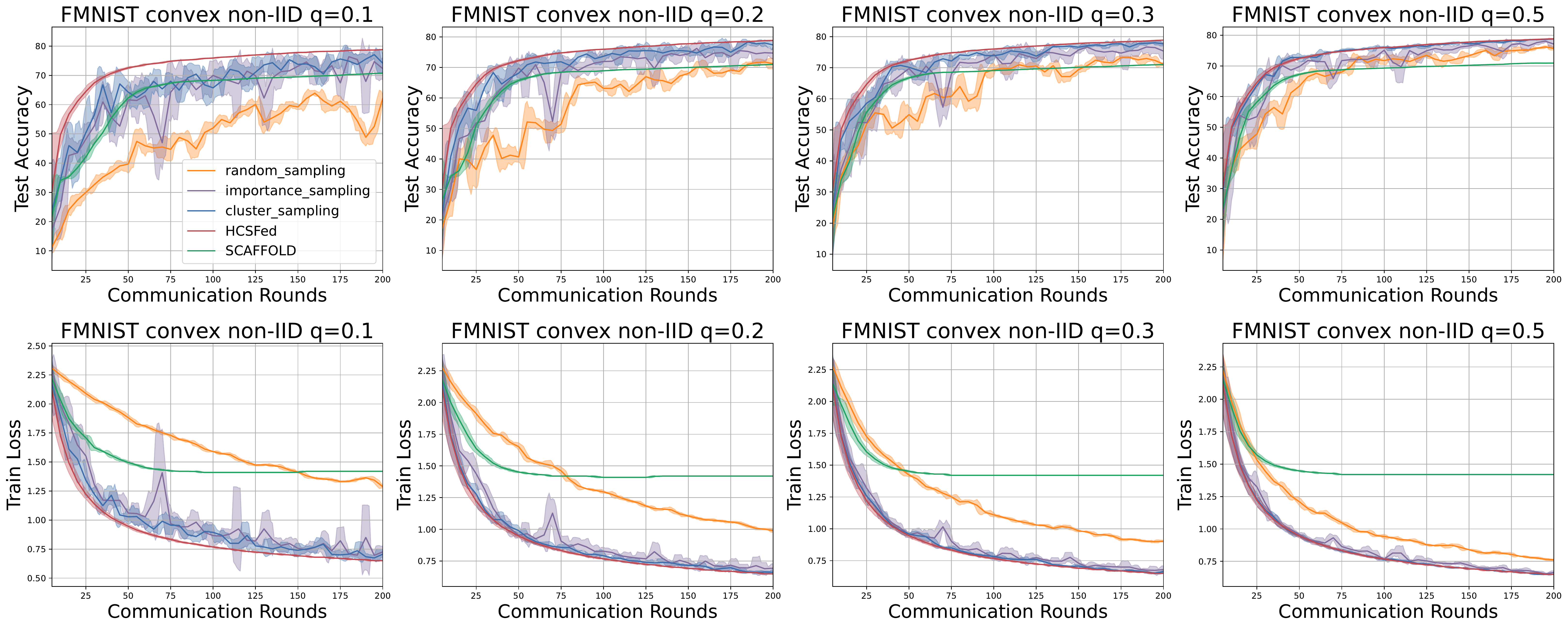}
    \caption{Impact of sampling ratio $q$ on the performance with convex model. We compare \texttt{HCSFed} with simple random sampling, importance sampling, cluster sampling, SCAFFOLD on FMNIST under non-IID, setting parameters $q\in \{0.1, 0.2, 0.3, 0.5\}$, $N=100$, $nSGD=50$, $\eta=0.01$, $B=50$.}
    \label{app:fig2}
    %\vspace{-2em}
\end{figure}
% \clearpage
\begin{figure}[!htbp]
    \centering
    \includegraphics[width=13.5cm]{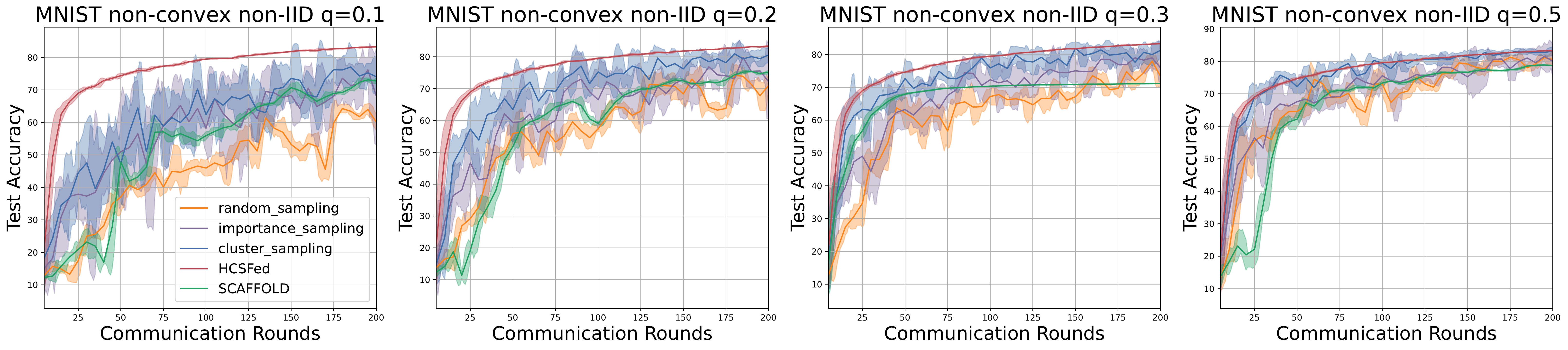}
    \caption{Impact of sampling ratio $q$ on the performance with non-convex model. We compare \texttt{HCSFed} with simple random sampling, importance sampling, cluster sampling, SCAFFOLD on MNIST under non-IID, setting parameters $q\in \{0.1, 0.2, 0.3, 0.5\}$, $N=100$, $nSGD=50$, $\eta=0.01$, $B=50$.}
    \label{app:fig3}
\end{figure}
\clearpage
\begin{figure}[!htbp]
    \centering
    \includegraphics[width=13.5cm]{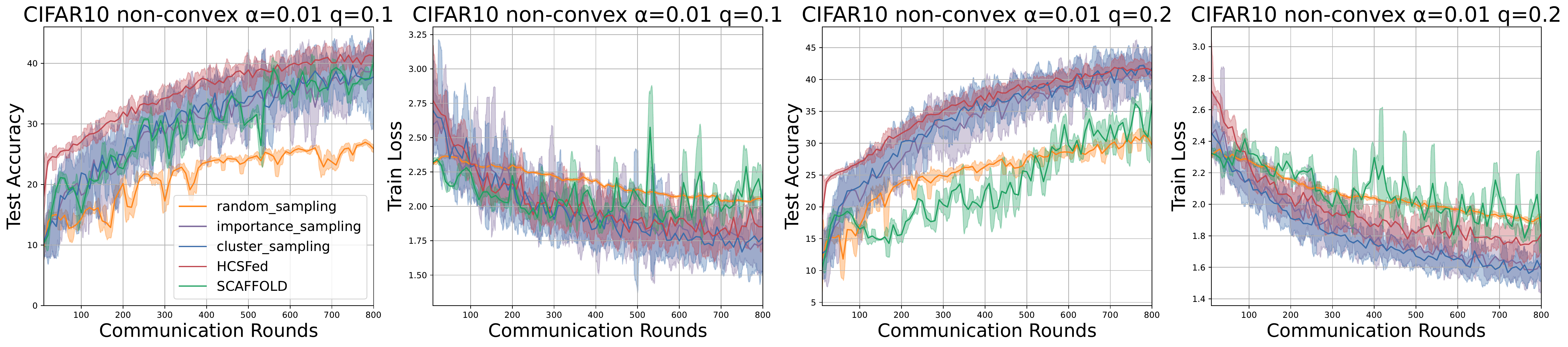}
    \caption{Impact of sampling ratio $q$ on the performance with non-convex model. We compare \texttt{HCSFed} with simple random sampling, importance sampling, cluster sampling, SCAFFOLD on CIFAR-10, using a Dirichlet Distribution with $\alpha=0.01$, setting parameters $q\in \{0.1, 0.2\}$, $N=100$, $nSGD=80$, $\eta=0.05$, $B=50$.}
    \label{app:fig4}
\end{figure}
\begin{figure}[!htbp]
    \centering
    \includegraphics[width=13.5cm]{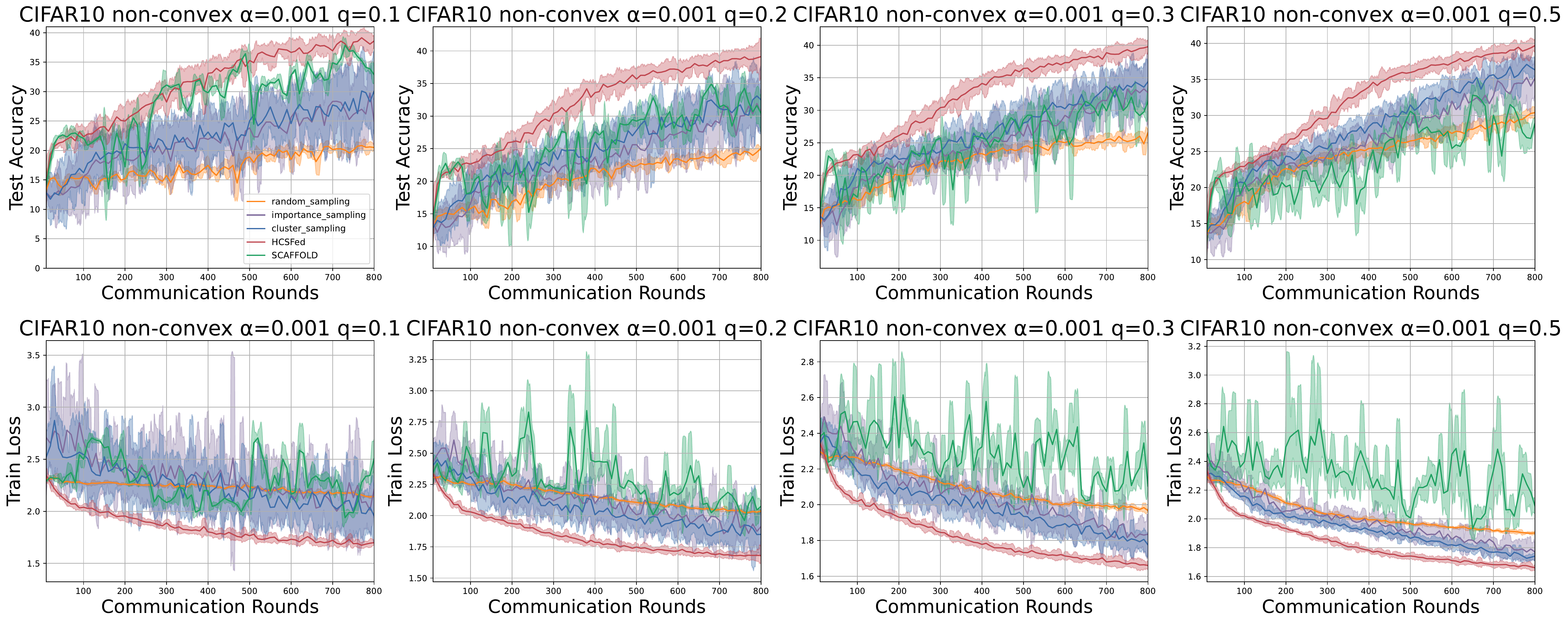}
    \caption{Impact of the heterogeneity on the performance with non-convex model. We compare \texttt{HCSFed} with simple random sampling, importance sampling, cluster sampling and SCAFFOLD on CIFAR-10, using a Dirichlet Distribution with $\alpha=0.001$, setting parameters $q \in \{0.1, 0.2, 0.3, 0.5\}$, $N=100$, $nSGD=80$, $\eta=0.05$, $B=50$.}
    \label{app:fig5}
\end{figure}
\subsection{Extra Experiments on FedNova}
% We carry out extra experiments on FedNova, a modified FL algorithm, to further verify the compatibility of our sampling scheme. We use all datasets mentioned above and take different distributions into consideration. As illustrated in Figure \ref{fig11}, our sampling scheme achieves superb performance on FedNova, especially under heterogeneity. 
\begin{figure}[!htbp]
    \centering
    % \setcaptionwidth{.91\textwidth} % 设置图片标题宽度
    \centering
    \includegraphics[width=13.5cm]{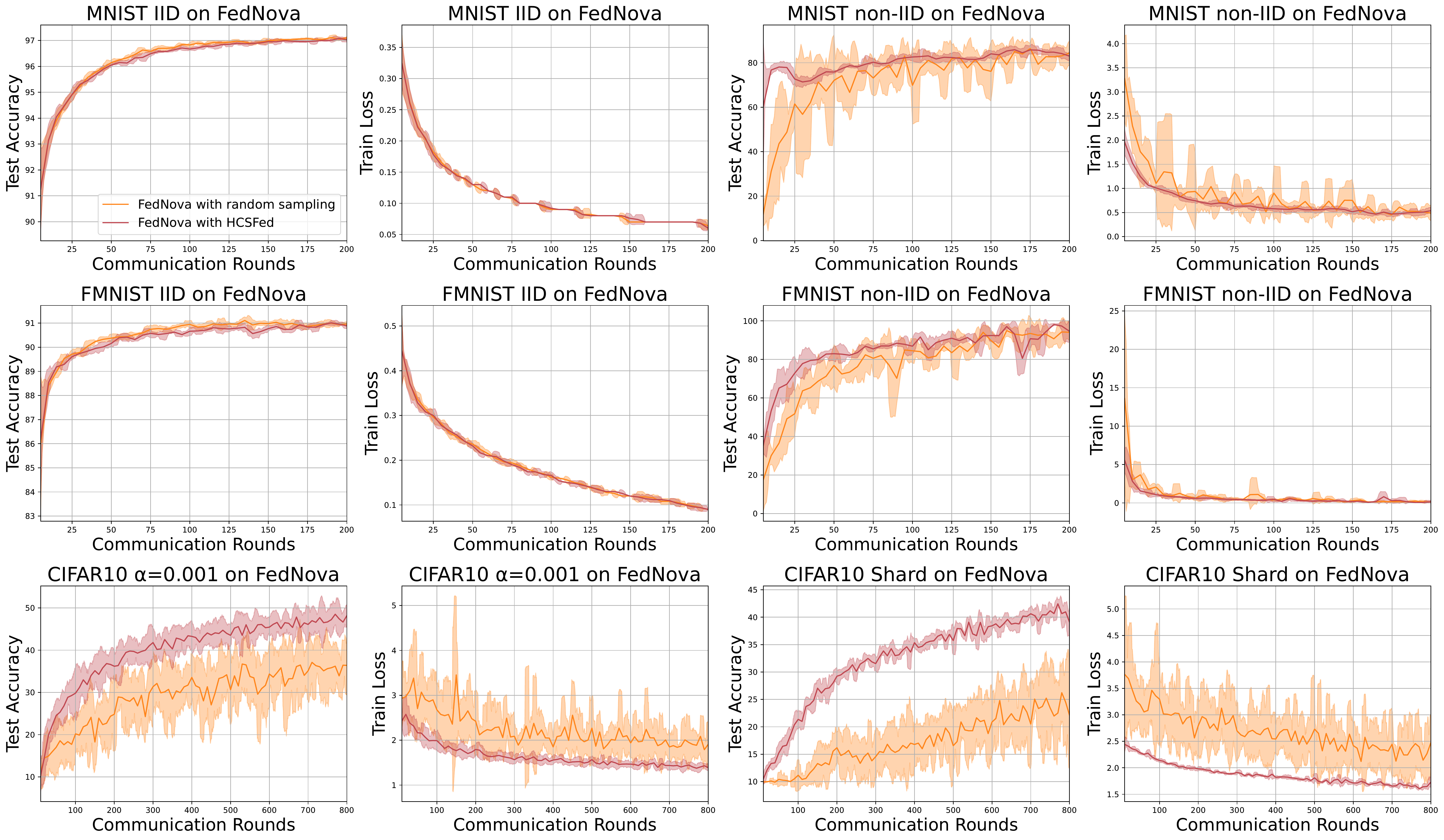}
    \caption{Results on the MNIST, FMNIST and CIFAR10 under FedNova. We compare \texttt{HCSFed} with simple random sampling, using a Dirichlet Distribution with $\alpha =0.001$ and Shard, setting parameters $q=0.1$, $N=100$, $nSGD=50$ for MNIST and FMNIST, $nSGD=80$ for CIFAR-10, $\eta=0.05$, $B=50$.}
    \label{app:fig6}
\end{figure}
\clearpage
% }
\end{document}